%% file: neurips_2026.tex
\documentclass{article}

\usepackage[main,final]{neurips_2026}

\usepackage[utf8]{inputenc} 
\usepackage[T1]{fontenc}    
\usepackage{hyperref}       
\usepackage{url}            
\usepackage{booktabs}       
\usepackage{amsfonts}       
\usepackage{nicefrac}       
\usepackage{microtype}      
\usepackage{xcolor}         

\usepackage{amsmath,amssymb} 
\usepackage{color}
\usepackage{algorithm}
\usepackage{algorithmic}
\usepackage{graphicx}
\usepackage{multirow}
\usepackage{subfigure}
\usepackage{comment}
\usepackage{stackengine}
\usepackage{colortbl}
\usepackage{amsthm}     
\usepackage{caption}
\usepackage{wrapfig}
\usepackage{physics}
\usepackage{amssymb}
\usepackage{amsmath}
\usepackage{mathtools}
\usepackage{cleveref}

\newtheorem{lemma}{Lemma}

\newtheorem*{proposition*}{Proposition}
\newtheorem{assumption}{Assumption}
\newtheorem{definition}{Definition}
\renewcommand{\algorithmicrequire}{ \textbf{Input:}} 
\renewcommand{\algorithmicensure}{ \textbf{Output:}} 

\newtheorem{theorem}{Theorem}
\newtheorem*{theorem*}{Theorem}

\makeatletter
\newcommand{\STATEx}{\item[]}
\makeatother

\newcommand{\bw}{\text{\boldmath{$w$}}}

\newcommand{\bH}{\text{\boldmath{$H$}}}

\newcommand{\bs}{\boldsymbol{s}}
\newcommand{\bu}{\boldsymbol{u}}

\newcommand{\bS}{\boldsymbol{S}}

\newcommand{\balpha}{\boldsymbol{\alpha}}

\newcommand{\bI}{\boldsymbol{I}}

\newcommand{\bA}{\boldsymbol{A}}

\newcommand{\cD}{\mathcal{D}}

\newcommand{\cL}{\mathcal{L}}

\newcommand{\cP}{\mathcal{P}}

\newcommand{\cA}{\mathcal{A}}

\newcommand{\cH}{\mathcal{H}}

\newcommand{\cO}{\mathcal{O}}

\newcommand{\train}{\mathrm{train}}
\newcommand{\val}{\mathrm{val}}

\title{BLADE: Scalable Bi-level Adaptive Data Selection for LLM Training}

\author{%
  Jiaxing Wang\thanks{Equal Contribution.}\hspace{1ex}$^{1}$, Deping Xiang\footnote[1]{}\hspace{1ex}$^{1}$, Jin Xu$^{2}$, Zirui Liu$^{3}$, Zicheng Zhang\thanks{Corresponding to Zicheng Zhang and Pengzhang Liu}~\hspace{0.8ex}$^{1}$, Guoqiang Gong$^{1}$, \\
  \textbf{\quad Jun Fang$^{1}$, Chao Liu$^{1}$, Pengzhang Liu\footnote[2]{}\hspace{1ex}$^{1}$, Tongxuan Liu$^{1}$, Ke Zhang$^{1}$, Qixiang Jiang$^{1}$} \\
  $^{1}$JD.com \quad $^{2}$ University of Oxford \quad $^3$Peking University \\
  \texttt{\{wangjiaxing41,xiangdeping1,zhangzicheng6,gongguoqiang1,fangjun8\}@jd.com} \\
  \texttt{\{liuchao397,liupengzhang,liutongxuan,zhangke323,jiangqixia\}@jd.com} \\
  \texttt{jin.xu@stats.ox.ac.uk \quad zirui.liu@pku.edu.cn.}
}

\begin{document}

\maketitle

\begin{abstract}
As Large Language Model (LLM) datasets scale to trillions of tokens, data selection has emerged as a critical frontier to filter out uninformative noise and construct adaptive learning trajectories. Beyond static heuristic filtering, advanced data selection methods for LLM training largely follow two paradigms, each with fundamental limitations. Influence-based methods provide principled bi-level objectives but require intractable inverse-Hessian computations, while excess-loss methods are computationally efficient but rely on a static reference model that becomes misaligned with the evolving proxy model during training. We propose BLADE (Bi-Level Adaptive Data sElection), a Hessian-free framework for data selection. BLADE reformulates the bi-level optimization problem underlying influence-based methods as a penalized single-level objective via Lagrange multipliers, avoiding inverse-Hessian computation while revealing a principled connection to excess-loss based data selection. The resulting objective recovers an excess-loss form but replaces the static reference model with a dynamic one that stays synchronized with training.  Theoretically, we prove that this penalized formulation guarantees first-order convergence. For efficient online batch selection, we instantiate BLADE as a memoryless randomized block-coordinate Frank-Wolfe algorithm. Extensive experiments show that BLADE consistently outperforms state-of-the-art data selection baselines, providing a practical recipe for LLM training.
\end{abstract}

\input{sections/introduction.tex}
\input{sections/methods.tex}
\input{sections/related.tex}

\input{sections/experiment.tex}
\input{sections/conclusion.tex}

\bibliographystyle{plain}
\bibliography{bibfile}

\appendix
\input{sections/appendix.tex}

\end{document}

%% file: sections/introduction.tex
\section{Introduction}
\label{sec:introduction}

The remarkable capabilities of Large Language Models (LLMs) are fundamentally driven by the scale and quality of their pre-training data. However, as available corpora reach trillion-token scales, naively training on the entirety of these datasets often yields diminishing returns, as low-quality, noisy, or highly redundant tokens can degrade model reasoning capabilities~\cite{zhou23lima,engstrom2024dsdm}. This makes data selection, the problem of choosing informative, high-signal training examples that best improve performance, a central problem in LLM training.

Existing model-aware data selection methods for LLM training largely follow two paradigms. Influence-based methods estimate how each training example affects validation loss, yielding a principled bi-level formulation of data selection~\cite{xia2024less,yu2024mates}. However, they require approximating inverse-Hessian terms arising from the bi-level objective, making them computationally prohibitive at LLM scale. Excess-loss methods avoid this cost by scoring examples through the loss gap between a proxy model and a reference model~\cite{mindermann2022rho,lin2024rho1}. Yet the reference model is typically trained \textit{a priori} and kept fixed, causing stale selection signals as the proxy model evolves. More broadly, the relationship between these paradigms remains unclear: influence-based methods are theoretically grounded but difficult to scale, while excess-loss methods are efficient but lack a principled connection to the bi-level data selection objective.

We introduce BLADE (Bi-Level Adaptive Data sElection), a framework that bridges the principled bi-level formulation of influence-based methods with the scalability of excess-loss-based approaches. Starting from the bi-level objective, BLADE avoids differentiating through the full inner training problem by incorporating the inner optimality condition into the outer objective via Lagrange multipliers. The resulting penalized formulation eliminates the need for inverse-Hessian approximation and reveals a principled connection to excess-loss formulation: data are scored by the discrepancy between a proxy model trained on selected training data and a reference model that also incorporates validation information. Unlike conventional excess-loss methods, BLADE does not use a reference model trained \textit{a priori} and kept fixed. Instead, it dynamically synchronizes the reference model with the evolving proxy model throughout training. This allows BLADE to retain the bi-level foundation of influence-based methods while matching the scalability and online efficiency of excess-loss-based data selection.

Our main contributions are summarized as follows:
\begin{itemize}
    \item We introduce BLADE, a Hessian-free reformulation of bi-level data selection that converts the bi-level objective into a penalized single-level problem via Lagrange multipliers, eliminating the need for inverse-Hessian approximation. BLADE enjoys first-order convergence guarantees.
    \item We show that BLADE bridges the two dominant paradigms for LLM data selection by deriving an excess-loss-style selection criterion from the bi-level objective and replacing the fixed reference model used in prior excess-loss methods with a dynamic reference model synchronized with the evolving proxy model.
    \item For efficient online batch selection, we instantiate BLADE as a memoryless randomized block-coordinate Frank-Wolfe method. Extensive experiments across various LLM training scenarios show that it consistently accelerates convergence and outperforms existing data selection baselines.
\end{itemize}

%% file: sections/methods.tex
\section{Methodology}
\label{sec:methodology}
In this section, we present BLADE, a Hessian-free framework for bi-level adaptive data selection. We start from a rigorous bi-level formulation of data selection and derive a penalized single-level objective, which we subsequently instantiate as an efficient online token-level selection algorithm. Figure~\ref{fig:computation_precdure} illustrates the overall procedure, and the complete workflow is provided in Algorithm~\ref{alg:workflow}. The notations used in this paper are summarized in Appendix~\ref{sec:summary_of_notations}.

\begin{figure*}[t]
\centering
\includegraphics[width=0.88\linewidth]{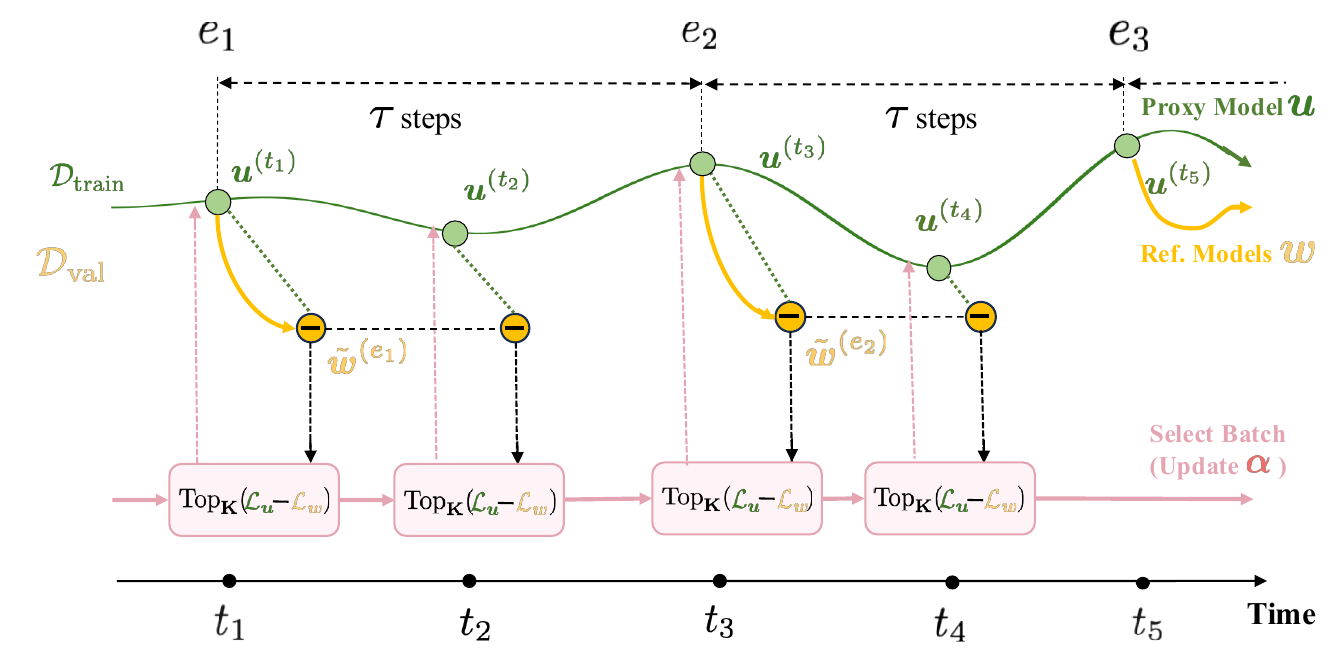}
\caption{The overall computation procedure of the BLADE framework. Twined proxy model (green) and reference model (orange) are used to guide the token selection (pink).}
\label{fig:computation_precdure}
\end{figure*}

\subsection{Problem Formulation}

Consider training an LLM on a dataset $\mathcal{D}_{train}=\{x_n\}_{n=1}^N$. In the context of dense autoregressive LLM training, we define the selection unit $n$ at the token level. Thus data selection refers to the problem of finding the optimal subset of informative tokens. We encode this decision using a continuous relaxation of the selection vector, defining the feasible set as $\mathcal{A} := \left\{\boldsymbol{\alpha}\in [0,1]^N  |  \sum_{n=1}^{N}\alpha_n=\gamma N \right\}$~\footnote{While data selection is inherently discrete, this continuous relaxation mathematically establishes a differentiable, gradient-based bi-level optimization objective. Crucially, as we will show in Section~\ref{subsec:blade}, our choice of optimization algorithm naturally returns discrete extreme points ($\balpha \in \{0, 1\}^N$), seamlessly bridging continuous theory with discrete data selection.}, where $\gamma$ is the keeping ratio.
For any selection vector $\balpha$, let $w(\balpha)$ denote the optimal model parameters obtained by minimizing the weighted training loss. Our goal is to minimize the validation loss $\mathcal{L}_{\text{val}}(w(\balpha))$ over $\balpha$, yielding the following bi-level optimization problem:
\begin{equation}
\small
\label{eq:bi-level}
    \min_{\balpha \in \mathcal{A}} \cL_{\val}(\bw(\balpha)) \ \ \ \   \text { s.t. } \bw(\balpha)\in\arg \min_{\boldsymbol{w}} \cL_{\train}(\balpha, \bw) :=  \frac{1}{\gamma N}\sum_{n=1}^N \alpha_n\mathcal{L}_{\text{train}}^{n}(\boldsymbol{w}).
\end{equation} 

Intuitively, the outer-level problem seeks the optimal token composition to minimize validation loss, constrained by model weights optimized on the inner-level reweighted training loss.

\subsection{BLADE: Hessian-Free Bilevel Data Selection}
\label{subsec:blade}
Solving the bi-level optimization problem~\eqref{eq:bi-level} is challenging due to its nested structure. By viewing the inner-level problem as a constraint and subsequently incorporating it into the outer-level problem as a Lagrangian penalty, Equation \eqref{eq:bi-level} can be reformulated as a single-level problem~\cite{shen23onpenalty,kwon2023f2sa,wang2025tandem}:
\begin{equation}
\label{eq:bi-level-penalty}
    \min_{\balpha\in\cA, \boldsymbol{w}} \cH_{\lambda}(\boldsymbol{\alpha}, \boldsymbol{w}):= \mathcal{L}_{\text{val}}( \boldsymbol{w})+ \lambda \left(\mathcal{L}_{\text{train}}( \boldsymbol{\alpha}, \boldsymbol{w})- \min_{\boldsymbol{u}} \mathcal{L}_{\text{train}}(\boldsymbol{\alpha}, \boldsymbol{u})\right).
\end{equation}
The change of variables in Equation~\eqref{eq:bi-level-penalty} is important. The auxiliary proxy $\bu$ tracks the lower-level optimum $\bw^*(\balpha)\in\arg\min_{\bw}\cL_{\train}(\balpha,\bw)$, whereas $\bw$ now denotes the reference model that tracks the optimum $\bw_\lambda^*(\balpha)$ of the penalized objective. Thus, the proxy is trained only on selected training tokens, while the reference additionally receives validation information. The difference $\cL_{\train}(\balpha,\bw)-\min_{\bu}\cL_{\train}(\balpha,\bu)$ measures violation of the inner optimality constraint. As justified by Proposition 3 in \citep{shen23onpenalty}, with a properly increasing penalty parameter $\lambda$, solutions of Equation~\eqref{eq:bi-level-penalty} approximate the original bi-level problem in Equation~\eqref{eq:bi-level} (see Appendix~\ref{sec:convergence}).

\paragraph{Algorithm Procedure}
Optimizing Equation~\eqref{eq:bi-level-penalty} involves coordinating three variables: the data selection indicator $\balpha$, a proxy model $\bu$, and a reference model $\bw$. Neither model is solved to exact optimality at every selection step: BLADE tracks the two optima online using one proxy update and a small, periodic block of reference updates. This partially trained regime is the mechanism that makes the formulation practical at LLM scale.

\textit{\textbf{Update on $\boldsymbol{u}$} (Proxy Model)}: Given current batch of data, determined by $\balpha^{(t)}$, we approximate the $\bu\in \arg\min_{\bu}\cL_{\train}(\balpha^{(t)}, \bu)$ in the penalization term with one gradient step:
\begin{equation}
\label{eq:updata_u}
    \boldsymbol{u}^{(t+1)}=\boldsymbol{u}^{(t)}-\eta_u \nabla_{{\boldsymbol{u}}} \mathcal{L}_{\text{train}}\left(\boldsymbol{\alpha}^{(t)}, \boldsymbol{u}^{(t)}\right)
\end{equation}
~(green line with arrow in Figure~\ref{fig:computation_precdure}). The proxy model is trained on the selected training data and maintained through the whole data selection process. 

\textit{\textbf{Update on $\boldsymbol{w}$} (Reference Model)}:
The reference model $\bw$ serves to capture signals from both the training and validation distributions. Based on the $\bw$-dependent terms in Equation~\eqref{eq:bi-level-penalty}, the update rule becomes:
\begin{equation}
\label{eq:update_w}
    \boldsymbol{w}^{(e)}_{k+1} = \boldsymbol{w}^{(e)}_{k} - \eta_{\boldsymbol{w}} \left( \nabla_{\boldsymbol{w}} \mathcal{L}_{\text{val}}\left(\boldsymbol{w}^{(e)}_k \right) + \lambda\nabla_{\boldsymbol{w}}\mathcal{L}_{\text{train}}\left(\boldsymbol{\alpha}^{(e \tau)}, \boldsymbol{w}^{(e)}_k \right)  \right),
\end{equation}

To avoid an overly erratic reference model, $\bw$ is updated periodically every $\tau$ steps of training $\bu$. We introduce an episode index $e = \lfloor t/\tau \rfloor$ to denote the current update cycle. To control the optimization disparity between $\bw$ and $\bu$, we synchronize their starting points by setting $\bw_0^{(e)} = \bu^{(e\tau)}$ at the beginning of each episode. After running Equation~\eqref{eq:update_w} for $K$ steps~(orange line in Figure~\ref{fig:computation_precdure}), the finalized reference model $\tilde{\bw}^{(e)} = \bw_K^{(e)}$ is utilized for data selection during the subsequent $\tau$ steps.

\textit{\textbf{Update on $\boldsymbol{\alpha}$} (Data Weights)}:
The term related to $\boldsymbol{\alpha}$ is $\lambda \left(\mathcal{L}_{\text{train}}( \boldsymbol{\alpha}, \boldsymbol{w})- \min_{\boldsymbol{u}} \mathcal{L}_{\text{train}}(\boldsymbol{\alpha}, \boldsymbol{u})\right)$. That says, the $\balpha$ update seeks a solution to Equation~\eqref{eq:bi-level-penalty} by finding data that aligns the training loss with the validation-informed reference model. Recall that $\mathcal{L}_{\text{train}}(\boldsymbol{\alpha}, \cdot)$, by definition, is the $\boldsymbol{\alpha}$ weighted token-wise loss, taking the gradient w.r.t $\balpha$ and applying Frank Wolfe update yields:

\begin{equation}
    \label{eq:update_alpha}
    \balpha^{(t+1)} = \balpha^{(t)} + \eta_{\balpha}(\bs^{(t)} - \balpha^{(t)}) \  \text{where} \ \bs = \arg \min _{\balpha \in \mathcal{A}}\left\langle \underbrace{ \mathcal{L}_{\text{train}}^{1:N} \left( \Tilde{\boldsymbol{w}}^{(e)}  \right)}_{\textbf{reference model}} -\underbrace{\mathcal{L}_{\text{train}}^{1:N}\left(\boldsymbol{u}^{(t)}\right)}_{\textbf{proxy model}} , \balpha\right\rangle
\end{equation}

Let $\Delta \in \mathbb{R}^N$ denote the hyper-gradient vector with respect to $\balpha$, where its $n$-th element is $\Delta_n = \mathcal{L}_{\text{train}}^n(\bw) - \mathcal{L}_{\text{train}}^n(\bu)$~\footnote{For simplicity, we omit the iteration subscripts of the reference model $\bw$ and proxy model $\bu$.}. The Frank--Wolfe oracle minimizes this gradient, so it selects the largest values of $-\Delta_n=\mathcal{L}_{\text{train}}^n(\bu)-\mathcal{L}_{\text{train}}^n(\bw)$: tokens on which the proxy has larger excess loss relative to the validation-informed reference.

\paragraph{Convergence Analysis} 
Because the objective is non-convex and constrained over $\balpha \in \mathcal{A}$, we evaluate its first-order stationary convergence using the Frank-Wolfe Gap, which serves as the standard convergence measure for conditional gradient methods in non-convex settings~\cite{simon2016convergefw,reddi2016sfwnoncovex}. The theoretical convergence of our method is summarized below.

\begin{theorem}[Informally]
\label{thm:convergence}
Under standard regularity assumptions (PL condition, smoothness, bounded Hessians, detailed in Appendix~\ref{sec:convergence}),  $\balpha^{(t)}$ obtained in the BLADE algorithm converges to a first-order stationary point of the bi-level objective~\eqref{eq:bi-level} in the rate of $\cO(T^{-\frac{1}{4}})$ by properly selecting $\lambda, K, \eta_{\balpha},\eta_{\bw}$, and $\eta_{\bu}$.

\end{theorem}
The proof is left in Appendix~\ref{sec:convergence}. It explicitly controls the tracking errors of the partially trained proxy and reference models; it does not require solving either model to convergence at each outer iteration. Theorem~\ref{thm:convergence} therefore grounds the loss-gap criterion as a first-order direction for the bi-level objective.

The Frank-Wolfe algorithm provides an intrinsic property to bridge our continuous optimization with the discrete reality of data selection: the Linear Minimization Oracle (LMO, solving for $\bs$ in Equation~\eqref{eq:update_alpha}) naturally projects the update onto the extreme points of the constraint set $\mathcal{A}$. Taking a greedy step size ($\eta_\alpha = 1$) yields a perfectly sparse, binary closed-form solution:
\begin{equation}
    \label{eq:update_alpha_closed_form}
    \alpha_n^{(t+1)} = I_{\text{Top-}\gamma N}(-\Delta_n^{(t)})
\end{equation}
meaning $\alpha_n^{(t+1)} = 1$ if the excess loss $-\Delta_n^{(t)}$ ranks in the Top-$\gamma N$, and $0$ otherwise, exactly recovering the discrete data-selection decision.

Frank--Wolfe is particularly well matched to the capped-simplex constraint. Its linear minimization oracle returns an extreme point of $\mathcal{A}$, which is exactly a binary Top-$\gamma N$ mask; with $\eta_{\balpha}=1$, no threshold, rounding, or projection is needed. An $\ell_1$ penalty would be constant under the fixed-budget constraint $\sum_n\alpha_n=\gamma N$ (and a relaxed-budget variant would introduce a tuning parameter and soft weights), while Euclidean projection onto $\mathcal{A}$ can return fractional coordinates that still require rounding. The closed-form Frank--Wolfe oracle is therefore both faithful to the discrete budget and inexpensive.

\paragraph{BLADE Admits Online Batch Selection} 
While Equation~\eqref{eq:update_alpha_closed_form} establishes the selection rule over the global dataset of size $N$, evaluating the full hyper-gradient $\Delta$ across billions of tokens at every step is infeasible. To transition to a scalable practical implementation, we optimize $\balpha$ with a memoryless randomized block-coordinate Frank--Wolfe method~\cite{lacoste-julien2013blockfw,reddi2016sfwnoncovex,hazan2012prejectionfree}. At each iteration, we uniformly sample a candidate batch $\mathcal{B}_{cand}^{(t)}$, effectively choosing a coordinate block of the high-dimensional vector $\balpha$. Applying Top-$K$ selection ($K = \gamma|\mathcal{B}_{cand}^{(t)}|$) within this block is the online realization of the full-coordinate criterion. Beyond computational efficiency, this online batch curation~\cite{mindermann2022rho,lin2024rho1,wang2024greats} aligns with the non-stationary dynamics of LLM training and continuously adapts the curriculum to the evolving proxy model.

Theorem~\ref{thm:convergence} analyzes the full-coordinate criterion and establishes what should be selected; the batched algorithm specifies how that criterion is applied at scale. A formal stochastic convergence result for the single-pass, randomized block-coordinate implementation is outside the present scope and is an important direction for future work. Empirically, Section~\ref{sec:experiments} shows that this batched realization remains effective throughout training and at larger scales.
Finally, to preserve the structural integrity and autoregressive dependencies of the attention context window during LLM training, tokens with $\alpha_n = 0$ are not physically discarded from the sequence. Instead, the binary decision acts as a token-level loss mask (as in RHO-1~\cite{lin2024rho1}). This masking mechanism ensures that BLADE can orchestrate a fine-grained, dynamic token curriculum without disrupting the sequential causality of the language model.

\subsection{BLADE Unifies and Improves Existing Data Selection Methods}
\label{subsec:unifies}
Our BLADE framework elegantly bridges the gap between the two dominant paradigms in data selection:  the theoretical rigor of influence-based methods~\cite{pan2024gdig,yu2024mates} and the computational efficiency of excess-loss-based methods~\cite{mindermann2022rho,lin2024rho1}. By inspecting their core selection criteria, we demonstrate how BLADE unifies these approaches while systematically addressing their respective limitations.

\begin{table}[h]
\centering
\caption{BLADE unifies the principled bi-level origin of the influence-based and the tractable excess-loss structure. By reformulating the bi-level objective with a penalty, BLADE bypasses the intractable Hessian required by the influence-based method, while employing a dynamic reference model to overcome the lagging issue of excess-loss based methods.}
\label{tab:criterion_comparison}
\begin{tabular}{@{}lc@{}}
\toprule
Method & Selection Criteria (Scoring) \\ \midrule
Excess-loss based &     $\underbrace{\mathcal{L}_{\text{train}}^{n}(\bar{\bw})}_{\text{Static Ref. model}} - \underbrace{\mathcal{L}_{\text{train}}^{n}(\bu^{(t)})}_{\text{Proxy model}}$                   \\ [4.2ex]
Influence-based &
$-\nabla_{\bw}\mathcal{L}_{\text{train}}^{n}(\bw)^{T} \bH^{-1} \nabla_{\bw}\mathcal{L}_{\text{val}}(\bw)$ \\  [1.2ex]
BLADE & $-\nabla_{\bw}\mathcal{L}_{\text{train}}^{n}(\bw)^{T} \bH^{-1} \nabla_{\bw}\mathcal{L}_{\text{val}}(\bw)$  \ $\Rightarrow$ \  $\underbrace{\mathcal{L}_{\text{train}}^{n}(\Tilde{\bw}^{(e)})}_{\textbf{Dynamic} \  \text{Ref. model}} - \underbrace{\mathcal{L}_{\text{train}}^{n}(\bu^{(t)})}_{\text{Proxy model}}$                \\ \bottomrule
\end{tabular}
\end{table}

\paragraph{Unifying Influence-Based Paradigms:} Influence-based methods prioritize data by estimating a sample's impact on validation loss. Mathematically, by invoking Cauchy's Implicit Function Theorem, evaluating the influence function is equivalent to computing a snapshot of the hyper-gradient of our exact bi-level objective (Equation~\eqref{eq:bi-level}) from a uniform initial state ($\alpha_n = \frac{1}{N}$):
\begin{equation*}
      \frac{\partial\mathcal{L}_{\text{val}}}{\partial\alpha_n} \propto -\nabla_{\bw}\mathcal{L}_{\text{train}}^n(\bw)^T \bH^{-1} \nabla_{\bw}\mathcal{L}_{\text{val}}(\bw)
\end{equation*}

where $\bH$ is the Hessian of the training objective ($\mathcal{L}_{\text{train}}$) (see Appendix~\ref{sec:hyper-gradient} for the full derivation). In this sense, existing influence-based methods are essentially attempting to approximate the solution to our bi-level formulation. However, exactly computing this inverse-Hessian product at scale is intractable, forcing prior work to rely on brittle approximation techniques~\cite{pan2024gdig,xia2024less} or auxiliary influence predictors~\cite{yu2024mates}. Instead, BLADE formally addresses the root bi-level optimization problem by reformulating it into a penalized single-level objective (Equation~\eqref{eq:bi-level-penalty}). This entirely bypasses the inverse-Hessian computation, providing a framework that retains the principled foundation of influence-based methods while achieving superior computational efficiency. 

\paragraph{Improving Excess-Loss Paradigms:} Excess-loss methods select training samples by computing the loss gap between a proxy model and a reference model trained solely on validation data. As outlined in Table~\ref{tab:criterion_comparison}, BLADE recovers this highly tractable scoring mechanism but fundamentally improves the reference architecture. Standard excess-loss approaches (like RHO-1) rely on a static reference model trained \textit{a priori}, which inherently stagnates and loses its discriminative power as the proxy model evolves during training. In contrast, BLADE introduces a dynamic reference model that periodically synchronizes with the proxy model. This dynamism ensures the reference remains competitive and accurately calibrated throughout the learning trajectory, mitigating the risk of lagging.

%% file: sections/related.tex
\section{Related Work}
\label{sec:related}
Data selection is attracting increasing attention as a pivotal strategy for developing LLMs with comprehensive and balanced capabilities. Our work builds upon two primary strands of research: data selection paradigms and bi-level optimization.

\paragraph{Data Selection} 
Early approaches to data selection primarily relied on static heuristics~\cite{raffel2020T5,rae2021gopher}, diversity sampling~\cite{tirumala2023d4}, and distribution matching~\cite{xie2023dsir} to filter out noisy or out-of-domain data. These methods, nevertheless, fail to align with the target tasks or the model's specific training dynamics. 
To achieve model-aware selection, recent advancements have predominantly bifurcated into influence-based and excess-loss-based paradigms. Influence-based methods leverage the influence function~\cite{koh2017influence} to quantify a sample's impact on validation loss. While theoretically principled, computing inverse Hessian-vector products is intractable at LLM scale. Consequently, methods like LESS~\cite{xia2024less} and G-DIG~\cite{pan2024gdig} resort to gradient matching or layer-wise approximations, while MATES~\cite{yu2024mates} trains a lightweight surrogate data influence predictor. LESS and G-DIG operate at sample level and require gradient features, making direct token-level application over billions of tokens prohibitively expensive; MATES is the stronger and more directly applicable influence-based baseline in our setting. Alternatively, excess-loss-based methods~\cite{mindermann2022rho} (and their token-level extension, RHO-1~\cite{lin2024rho1}) avoid this bottleneck by scoring data with the loss gap between a proxy and a well-trained reference model. Yet, their \textit{a priori trained}, static reference can become stale as the target model evolves. Token Cleaning~\cite{pang2025tokencleaning} also performs token-level selection using a self-evolving reference, but targets noisy-label removal for supervised fine-tuning in a two-stage clean-then-retrain pipeline. BLADE instead performs selection using a validation-informed reference derived from a bi-level objective.

\begin{wraptable}{r}{6cm}
\caption{Complexity of different methods.}
\label{tab:computational_complexity}
\centering
\begin{tabular}{@{}lc@{}}
\toprule
Method       & Training Complexity                                                                                                                                               \\ \midrule
Vanilla Train      & $C T $      \\
MATES       &  $\left( \gamma +  \frac{T_{\text{probe}} \times T_{\text{val}}}{\tau} \right) CT$  \\
RHO-1          & $ \left(\frac{2\gamma}{3} + \frac{2}{3}\right) C T $ \\
BLADE     & $ \left(\frac{2}{3} \gamma + \frac{\gamma K}{\tau}+ \frac{2}{3} \right)C T $ \\ \bottomrule
\end{tabular}
\end{wraptable}

Comparison of the computational complexity of these strategies is given in Table~\ref{tab:computational_complexity}~\footnote{We assume backward takes 2$\times$ computation as the forward process.}, where $C$ is the complexity of one step training of a model. $T$ is the update number of $\boldsymbol{\alpha}$. We explicitly unpack the dominant overhead of MATES as $\frac{T_{\text{probe}} \times T_{\text{val}}}{\tau} $, which represents the multiplicative cost of oracle data influence collection. Specifically, $T_{\text{probe}}$ denotes the number of probing samples, and $T_{\text{val}}$ is the number of steps to evaluate each sample's influence on the entire validation set. This nested complexity arises because MATES requires a full pass over the validation data to label its surrogate model for every probing sample, making it inapplicable for token-level selection. Typically, for relatively large validation data, the overall computational cost of these methods is of the order: MATES > BLADE $\approx$ RHO-1. 

\paragraph{Bi-level Optimization}

Bi-level optimization provides a rigorous mathematical framework in many scientific disciplines, yet solving it remains challenging due to the nested dependencies between the upper and lower-level problems. Standard gradient-based bi-level solvers~\cite{dagreou2022bilevel,choe2024metalearning,shaban2019tuncated,grazzi2023bilevel,chen2021tight,hong2023twotimescale} require estimating implicit gradients, introducing Hessian-vector products (HVPs) that are computationally prohibitive for large-scale neural networks. To mitigate this, recent works like BLISS~\cite{hao2026bliss} employ very small proxy models for HVP estimation, while PBCS~\cite{zhou2022pbilevel} utilizes policy gradients to bypass HVPs entirely, though it requires training the inner problem to convergence at every outer step. More recently, penalty methods~\cite{ye2022bome,shen23onpenalty,kwon2023f2sa,kwon2024onpenalty} reformulate inner optimality as a penalty and enable purely first-order optimization. BLADE addresses the substantially higher-dimensional problem of selection over billions of individual tokens and introduces periodic proxy--reference synchronization, which stabilizes the penalty dynamics and permits a moderate practical penalty. Thus, BLADE adapts first-order penalty-based bi-level optimization to scalable online token selection.

%% file: sections/experiment.tex
\section{Experiments}
\label{sec:experiments}
In this section, we compare BLADE to competitive algorithms in Section~\ref{subsec:compare_sota}. We then analyze the effectiveness of each design ingredient in Section~\ref{subsec:analysis}.

\subsection{Experimental Setup}
\label{subsec:exp_setup}

We consider two distinct continual pre-training scenarios: domain specialization (domain-shift) and general capability enhancement. Complete hyper-parameter settings are provided in Appendix~\ref{sec:hyper_params}.

\paragraph{Domain-Shift pre-training:}  To evaluate BLADE's ability to drive domain specialization, we follow~\cite{lin2024rho1} and utilize a blend of MetaMath~\cite{yu2024metamath} and Mammoth~\cite{yu22024mammoth}~\footnote{This is a subset of the validation data used in RHO-1~\cite{lin2024rho1}; the remaining validation data used by RHO-1 are not open-sourced.} as the validation set. The candidate training pool is a 5B randomly sampled subset of OpenWebMath~\cite{paster2024openwebmath} ($\sim$14B math-related web tokens). We deliberately employ foundational base models TinyLlama-1.1B~\cite{zhang2024tinyllama}, Llama2-7B~\cite{touvron2023llama2}, and Llama2-13B rather than heavily optimized models, so that saturated capabilities do not mask the effect of data selection. Following~\cite{lin2024rho1}, selection is executed at the token level with a keeping ratio of $\gamma = 0.6$. The reference model $\bw$ is synchronized periodically ($\tau = 1000$). Evaluations are conducted across various math benchmarks: \textbf{GSM8K}~\cite{cobbe2021gsm8k}, \textbf{MATH}~\cite{hendrycks2021math}, \textbf{SVAMP}~\cite{patel2021svamp}, \textbf{ASDIV}~\cite{miao2020asdiv}, \textbf{MAWPS}~\cite{kedziorski2016mawps}, \textbf{TabMWP (TAB)}~\cite{lu2023tab}, \textbf{MathQA (MQA)}~\cite{amini2019mathqa}, \textbf{MMLU-STEM}~\cite{hendrycks2020mmlu}, and \textbf{SAT}~\cite{azerbayev2023sat}. This evaluation encompasses diverse question types such as multiple-choice and open-ended questions. Experiments are conducted on eight NVIDIA Hopper H-100s.

\paragraph{General pre-training:} 
To verify that BLADE genuinely improves broad capabilities, we utilize OpenHermes-2.5~\cite{teknium2023openhermes} as the validation set. The candidate training pool is a 5B blend of OpenWebMath~\cite{paster2024openwebmath} and SlimPajama~\cite{soboleva2023slimpajama}. We evaluate the final trained proxy model on benchmarks prioritizing diversity of capabilities and formats, including \textbf{MMLU}~\cite{hendrycks2020mmlu}, \textbf{Hellaswag}~\cite{zellers2018hellaswag}, \textbf{OpenBookQA}~\cite{2018openbookqa}, \textbf{WinoGrande}~\cite{sakaguchi2020winogrande}, \textbf{ARC-Easy} and \textbf{ARC-Challenge}~\cite{clark2018arc}, \textbf{BoolQ}~\cite{clark2019boolq} and \textbf{PIQA}~\cite{bisk2020pika}. Following~\cite{lin2024rho1}, the evaluation is conducted using lm-evaluation-harness~\cite{gao2024lmevalharness}.

\subsection{Comparisons with Competitive Baselines}
\label{subsec:compare_sota}
We benchmark BLADE against representative baselines: \textbf{Base} (the untrained original model), \textbf{Base-CT} (naive continual pre-training on the full candidate set), \textbf{Random} (training on randomly sampled data), \textbf{MATES}~\cite{yu2024mates} (an influence-based method with an influence predictor), and \textbf{RHO-1}~\cite{lin2024rho1} (a competitive token-level excess-loss method). These methods directly represent the two paradigms that BLADE connects. For the baselines, we use the published MATES implementation~\footnote{\textbf{MATES:} https://github.com/cxcscmu/MATES.} and our implementation of RHO-1.

\paragraph{Domain-Shift Pre-training} 
During evaluation, the models are prompted with few-shot chain of thought (CoT) as in~\cite{lin2024rho1}. Table~\ref{tab:sota_math} reports results at three model scales in a single comparison. On TinyLlama-1.1B, BLADE achieves 29.7 average accuracy, surpassing RHO-1 by 4.0 points and Base-CT by 13.0 points. The advantage persists as the backbone grows: BLADE improves the average accuracy from 41.1 to 43.8 over RHO-1 on Llama2-7B and from 52.5 to 54.4 on Llama2-13B. BLADE therefore attains the highest average accuracy at all three scales. SAT contains only 32 questions and consequently exhibits high variance across individual predictions.

On TinyLlama-1.1B, BLADE (494.3 min) takes longer than naive continual training (314.6 min), a necessary trade-off to break the performance ceiling. Yet, among model-aware methods, BLADE adds only 62.8 minutes over the static RHO-1 and is faster than MATES. To prevent doubling peak VRAM, BLADE employs an offloading mechanism, periodically loading the reference model $\bw$ to pre-compute scores for $\tau$ steps before unloading it, thereby maintaining the peak VRAM of standard single-model training.

\paragraph{General Pre-training}
The evaluation results for the general continual pre-training scenario are summarized in Table~\ref{tab:sota_general}. All methods are evaluated over three runs. Due to limited space, the main table reports only the mean accuracy for each benchmark. The corresponding standard deviations are reported in Appendix~\ref{sec:general_multiseed}. BLADE outperforms competing data selection baselines in 7 out of 8 diverse downstream tasks and achieves an average accuracy of 51.88$\%$, a 1.23 point improvement over RHO-1.

The advantage persists when more candidate data are processed. With a larger 10B token candidate budget, BLADE trains on 6B retained tokens and reaches 52.37 average accuracy, outperforming RHO-1 by 0.99 points (See Table~\ref{tab:scale_general_10b} in Appendix~\ref{sec:larger_tokenexps}). Checkpoint-level evaluation further shows that BLADE overtakes RHO-1 after 2B processed candidate tokens and remains the strongest method shown throughout the remainder of training. The complete trajectories are provided in Appendix~\ref{sec:general_dynamics}.

\begin{table}[]
\caption{Comparison of CoT reasoning results for math pre-training with TinyLlama-1.1B, Llama2-7B, and Llama2-13B. Results on nine downstream tasks, average accuracy, and wall-clock training time are reported. $\dagger$ denotes results from our implementation. For a fair comparison, the RHO-1 reference model has the same size as its proxy model. $*$ denotes a sample-level selection method.}
\label{tab:sota_math}
\setlength\tabcolsep{2.8pt}
\small
\centering
\begin{tabular}{@{}lcccccccccccc@{}}
\toprule
               & \multicolumn{1}{l|}{\begin{tabular}[c]{@{}l@{}}Train\\ Toks.\end{tabular}} & GSM8K & MATH & SVAMP & ASDiv & MAWPS & TAB  & MQA  & \begin{tabular}[c]{@{}c@{}}MMLU\\ STEM\end{tabular} & SAT  & Avg. & Time$_\text{(min)}$ \\ \midrule
\multicolumn{13}{c}{\textit{\textbf{Smaller Models :  TinyLlama-1.1B}}}                                                                                                                                                               \\ \midrule
Base      & \multicolumn{1}{c|}{-}                                                    & 2.9   & 3.2  & 11.0  & 18.1  & 20.4  & 12.5 & 14.6 & 16.1                                                & 21.9 & 13.4 &  - \\
Base-CT   & \multicolumn{1}{c|}{5B}                                                  & 3.7   & 3.6  & 17.6  & 27.8  & 35.0    & 17.1 & 9.1   &  18.0                                                & 18.8 & 16.7  &  314.6 \\
Random   & \multicolumn{1}{c|}{3B}                                                  & 3.6   & 3.6  & 14.6  & 23.7  & 30.1  & 12.1 & 11.4 & 18.8                                               & 14.7 & 14.7 &  188.8 \\
MATES$^*$ & \multicolumn{1}{c|}{3B}                                                   & 4.2  & 4.8  & 17.4   &  28.1   & 38.3   & 16.0  & 11.9  & 13.6                                                 &  \textbf{25.0}    & 17.7  & 831.5 \\
RHO-1$\dagger$ & \multicolumn{1}{c|}{3B}                                                   & 14.6  & 4.8    & 34.9  & 46.7   &  61.1   & 17.0   & 16.0   &  20.8                                                  & 15.6  &  25.7   &  431.5 \\
BLADE         & \multicolumn{1}{c|}{3B}                                                   & \textbf{20.8}  & \textbf{8.6}  & \textbf{35.7}  & \textbf{48.9}  & \textbf{66.5}  & \textbf{21.9} & \textbf{17.2} & \textbf{22.4}                                                & \textbf{25.0} & \textbf{29.7} & 494.3\\ \midrule
\multicolumn{13}{c}{\textit{\textbf{Larger Models :  Llama2-7B}}}                                                                                                                                                                  \\ \midrule
Base      & \multicolumn{1}{c|}{-}                     &  14.0     &  3.6     &  39.5    &    51.7    &   63.5    &   30.9   &   12.4   &    32.7                               &   34.4    &   31.4   & - \\
Base-CT   & \multicolumn{1}{c|}{5B}                    &    18.1     &  4.4     &  36.4    & 52.7      &  66.1      & 37.3   &   18.5    &   33.4                        & 37.5     &  33.8    &  1670.3 \\
Random   & \multicolumn{1}{c|}{3B}                     & 17.3   &  5.2    &  36.9    &   51.9    &  64.2   &  40.2   &  17.1  & 32.9                          &   25.0   &  32.3 &  1011.0 \\
MATES$^*$ & \multicolumn{1}{c|}{3B}                                                   & 16.5  & 4.4  & 35.5   & 51.6   & 65.7   & 38.6  & 15.8  & \textbf{33.6}                                                 &  \textbf{38.6}    & 33.4   & 3274.7 \\
RHO-1$\dagger$ & \multicolumn{1}{c|}{3B}             &   33.0    &   7.6   &   57.9    &    60.6    &   78.9     &  48.4     &   24.9    &     30.1           &  28.1    &    41.1   &   2140.7 \\
BLADE         & \multicolumn{1}{c|}{3B}              &    \textbf{40.4}   &   \textbf{9.8}     &     \textbf{60.8}    &   \textbf{66.1}      &    \textbf{84.9}   &  \textbf{49.6}     &   \textbf{27.3}     &    33.1       &  21.9      &    \textbf{43.8}   &  2375.1 \\ \midrule
\multicolumn{13}{c}{\textit{\textbf{Larger Models :  Llama2-13B}}}      \\ \midrule
Base & \multicolumn{1}{c|}{--} & 26.5 & 3.8 & 44.0 & 58.6 & 75.2 & 47.9 & 18.2 & 37.3 & 31.2 & 38.1 & - \\
Base-CT & \multicolumn{1}{c|}{5B} & 27.4 & 9.4 & 54.3 & 61.8 & 76.0 & 40.3 & 29.5 & 44.4 & \textbf{50.0} & 43.7 & 2813.6 \\
Random & \multicolumn{1}{c|}{3B} &  25.7   &  9.6   &  51.2  &   60.1   &   76.0   &   47.1   &  24.9   &  43.5  &  46.9   &  42.8  & 1693.6   \\
RHO-1$\dagger$ & \multicolumn{1}{c|}{3B} & 56.7 & 15.2 & \textbf{68.8} & 73.1 & 89.8 & 58.1 & 32.1 & 41.6 & 37.5 & 52.5 & 3635.2 \\
BLADE & \multicolumn{1}{c|}{3B} & \textbf{57.0} & \textbf{19.0} & 68.3 & \textbf{74.1} & \textbf{90.0} & \textbf{59.9} & \textbf{41.9} & \textbf{44.8} & 34.4 & \textbf{54.4} & 4531.2 \\ \bottomrule
\end{tabular}
\end{table}

\begin{table}[]
\caption{Comparison for the TinyLlama-1.1B model on general continual pre-training. Accuracy (\%) on eight downstream tasks and the average are reported. $\dagger$ denotes results from our implementation.}
\label{tab:sota_general}
\small
\begin{tabular}{@{}lc|ccccccccc@{}}
\toprule
             & \multicolumn{1}{l|}{\begin{tabular}[c]{@{}l@{}}Train \\ Toks.\end{tabular}} & MMLU  & \begin{tabular}[c]{@{}c@{}}Hella\\ swag\end{tabular} & OBQA  & \begin{tabular}[c]{@{}c@{}}Wino\\ Grande\end{tabular} & ARC-C                        & ARC-E & BoolQ & PIQA  & Avg.   \\ \midrule
Base    & -                                                                           & 25.18 & 59.23                                                & 35.47 & 59.07                                                 & 30.86                        & 55.68 & 57.81 & 73.18 & 49.56 \\
Base-CT & 5B                                                                          & 25.62 & 58.28                                                & 34.33 & 58.72                                                 & 30.86                        & 55.85 & 61.91 & 72.10 & 49.71 \\

Random       & 3B                                                                          & 25.31 & 58.43                                                & 34.67 & 59.84                                                 & 31.07                        & 55.49 & 60.24 & 71.78 & 49.60 \\
MATES$^*$        & 3B                                                                          & 25.20 & 58.48                                                & 35.27 & 60.16                                                 & 31.57                        & 57.48 & 63.41 & 72.32 & 50.49 \\

RHO-1$\dagger$ & 3B                                                                        & 25.99 & 59.47                                                & \textbf{35.80} & 60.43                                                 & 31.86                        & 56.11 & 63.20 & 72.32 & 50.65 \\
BLADE       & 3B                                                                          & \textbf{27.95} & \textbf{60.28}                                                & 34.93 & \textbf{60.59}                                                 & \textbf{32.82}                        & \textbf{58.73} & \textbf{66.34} & \textbf{73.34} & \textbf{51.88} \\ \bottomrule
\end{tabular}
\end{table}

\subsection{Analysis}
\label{subsec:analysis}
To evaluate the effectiveness of each design component, we conduct ablation experiments. The model used is the TinyLlama-1B~\cite{zhang2024tinyllama} unless specified and we test on the domain-shift pre-training scenario. Due to limited space, we focus on the effect of the dynamic reference model, the effect of reference model update interval $\tau$ and the keeping ratio $\gamma$, and the synchronization mechanism of $\bw$ and $\bu$. We also discuss the reusability of the learned data orchestration on models of larger sizes in Appendix~\ref{sec:additional_exps}.

\paragraph{The Effect of the dynamic reference model}
\begin{figure}[t]
\centering
\begin{minipage}{0.345\linewidth}
    \centering    
    \subfigure{
    \label{fig:oracle_loss}
    \includegraphics[width=1.0\textwidth]{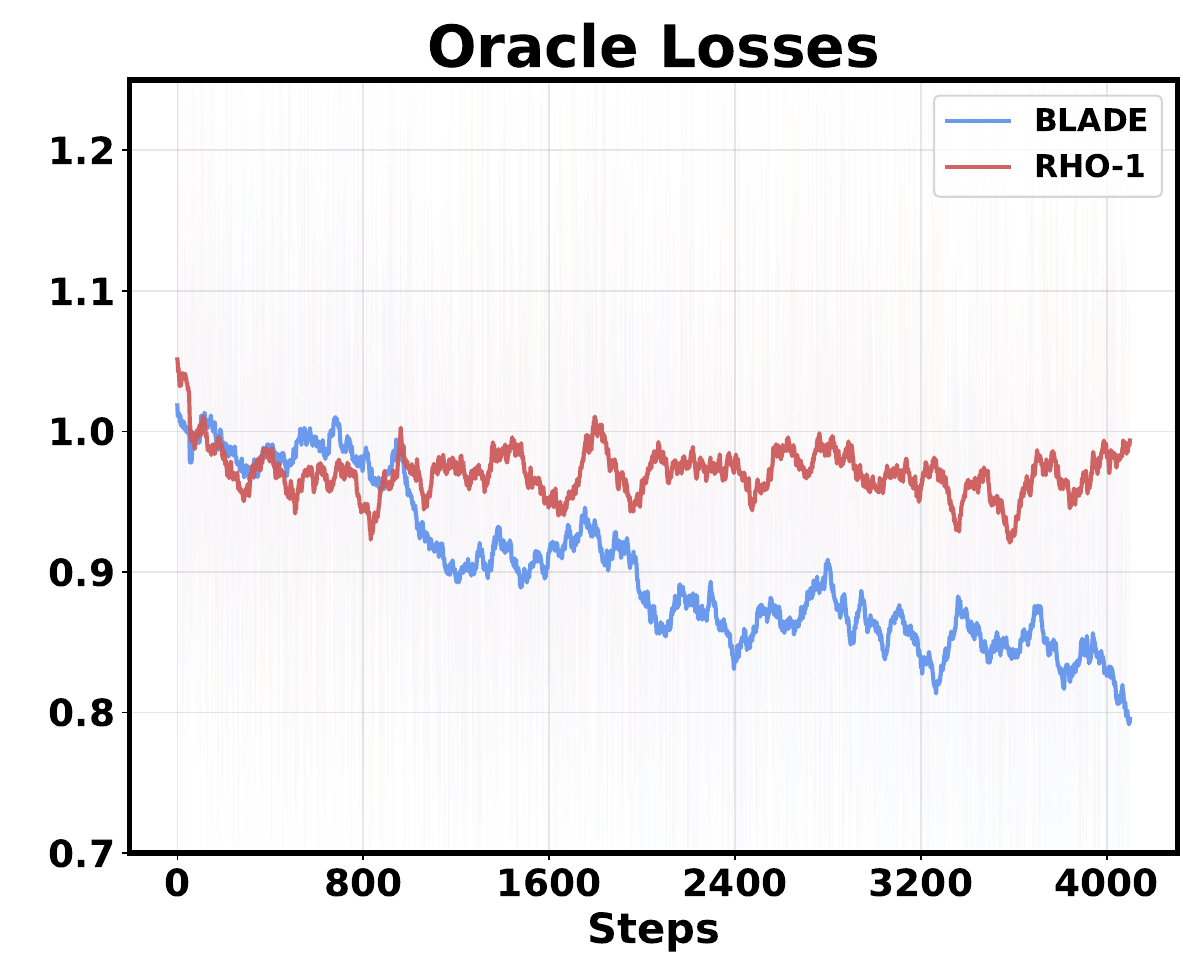}}
\end{minipage}
\begin{minipage}{0.315\linewidth}
    \centering
    \subfigure{	
    \label{fig:tau_effect}
    \includegraphics[width=0.98\linewidth]{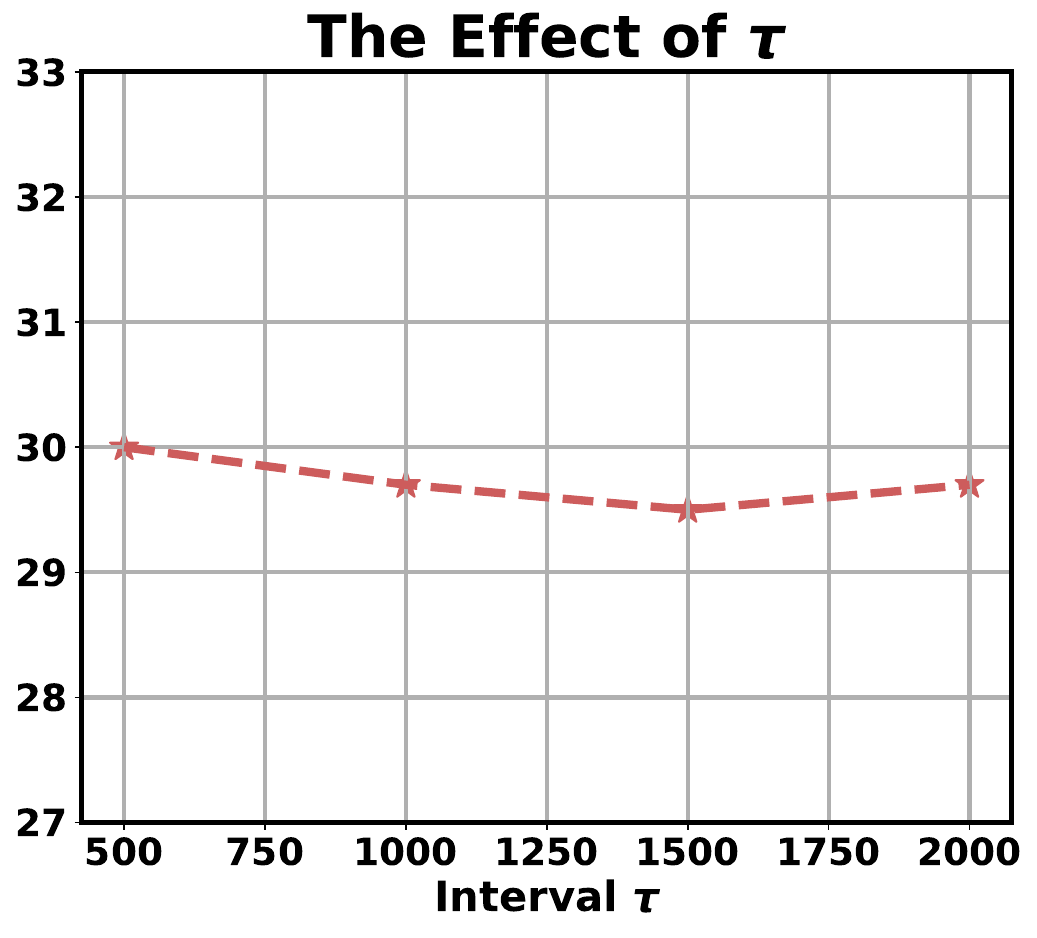}}
\end{minipage}
\begin{minipage}{0.315\linewidth}
    \centering
    \subfigure{	
    \label{fig:gamma_effct}
    \includegraphics[width=1.0\linewidth]{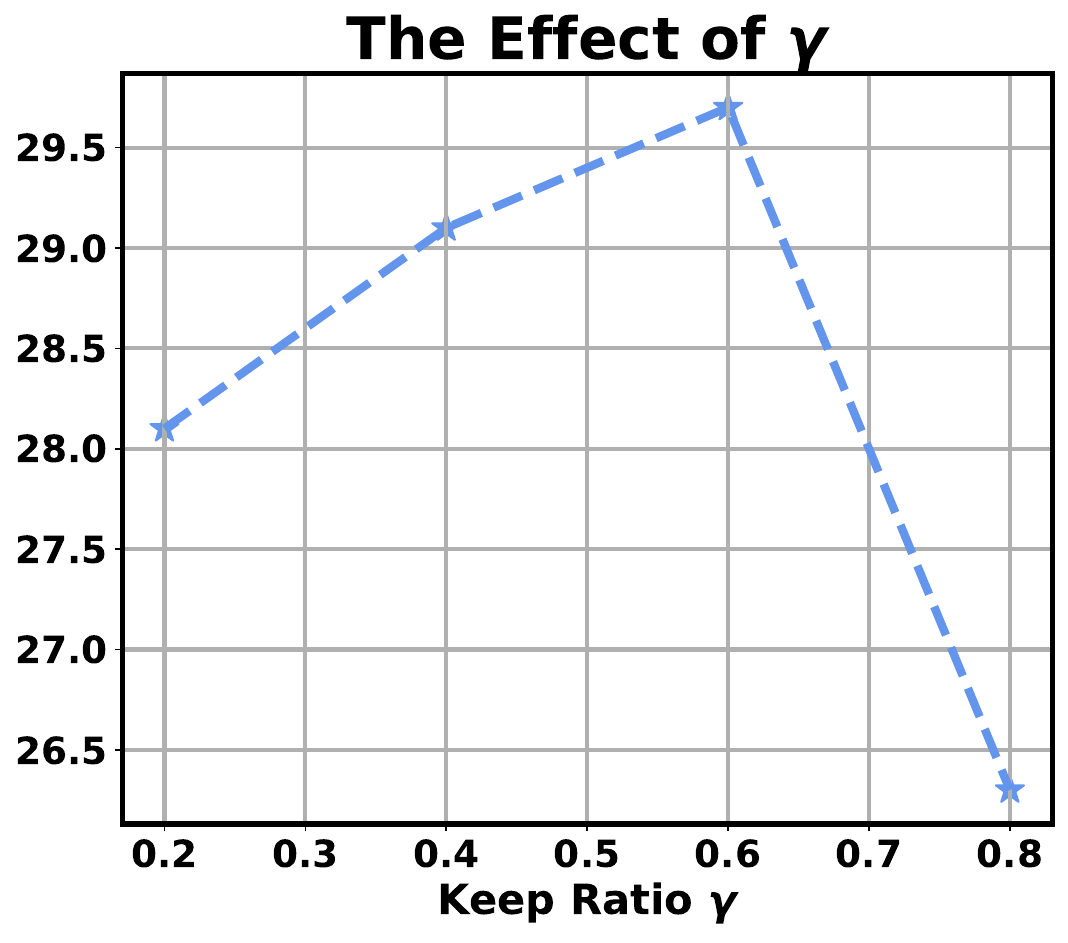}}
\end{minipage}
\caption{(a) The mean Llama2-70B oracle loss on the selected tokens. (b) The averaged accuracy when the reference model update interval $\tau$ changes. (c) The averaged accuracy when the keeping ratio $\gamma$ changes.}
\label{fig:mixture_ratio_evolve}
\end{figure}

\begin{figure}[t]
\centering
\includegraphics[width=0.98\linewidth]{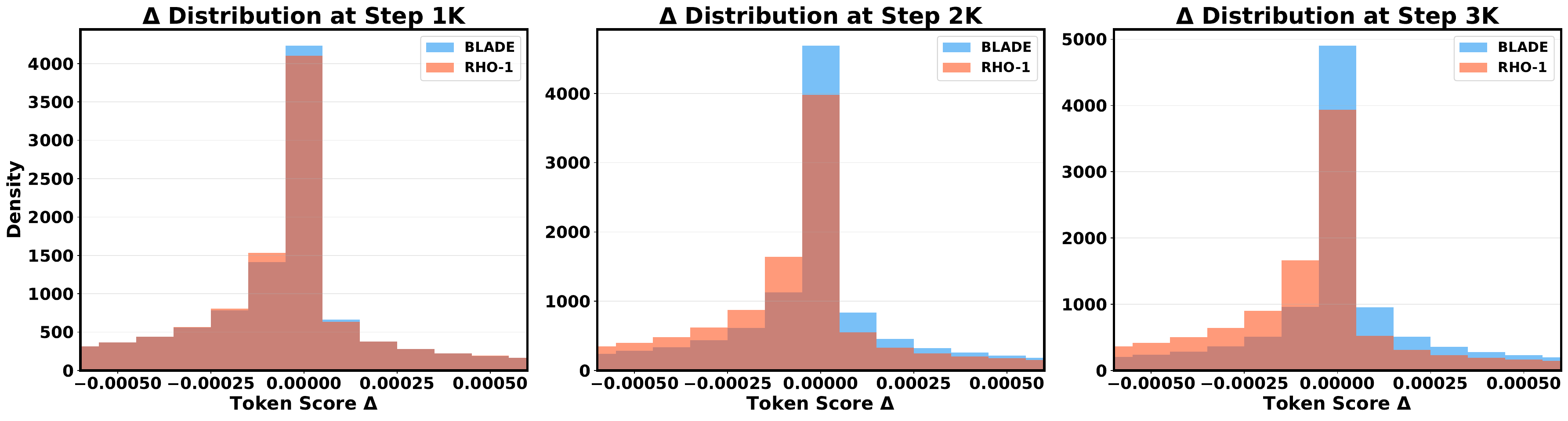}
\caption{Token-level score ($\Delta$) distributions at steps 1K, 2K, and 3K. While RHO-1's scores collapse toward non-positive values due to its static reference, BLADE's dynamic update maintains a pronounced right tail ($\Delta > 0$). This preserves BLADE's discriminative power to consistently select tokens with positive marginal utility.}
\label{fig:delta_distribution}
\end{figure}

To empirically validate the dynamic reference mechanism, we trace the selected tokens using a 70B oracle loss (Figure~\ref{fig:oracle_loss}) and monitor the evolution of their token-level excess loss ($\Delta$) distributions (Figure~\ref{fig:delta_distribution}). Initially (0–1K steps), both RHO-1 and BLADE exhibit high oracle loss as they aggressively select domain-rich web data to bridge the domain gap. The critical divergence occurs at Step 1000, coinciding with BLADE's first reference model update. As the proxy model quickly masters foundational domain basics, RHO-1's static reference lags. Figure~\ref{fig:delta_distribution} illustrates RHO-1's $\Delta$ distribution collapsing toward zero and negative values, losing its discriminative power and causing its oracle loss to stagnate. Conversely, BLADE recalibrates the reference against the proxy's growing capabilities, restoring a pronounced right-tail distribution ($\Delta > 0$). Guided by this calibrated signal, BLADE continuously targets tokens with genuine positive marginal utility, seamlessly shifting its curriculum from raw vocabulary acquisition to isolating cleaner and more structured reasoning tokens, as evidenced by the descent in its oracle loss.

\paragraph{The Effect of $\tau$ and $\gamma$:}

We systematically evaluate the impact of the reference update interval ($\tau \in \{500, 1000, 1500, 2000\}$). As shown in Figure~\ref{fig:tau_effect}, downstream performance remains highly consistent across all tested values. This insensitivity carries significant practical implications: practitioners can safely adopt a larger $\tau$ (e.g., 2000) to reduce the computational overhead of maintaining the dynamic reference without sacrificing selection quality. For the keeping ratio, we evaluate BLADE across ($\gamma \in \{0.2, 0.4, 0.6, 0.8\}$). As illustrated in Figure~\ref{fig:gamma_effct}, the model's average accuracy first increases as $\gamma$ goes up to 0.6. Nevertheless, increasing the budget to $\gamma = 0.8$ degrades performance, validating our core premise: an overly high keeping ratio inevitably forces the inclusion of uninformative or redundant noise, degrading the model's capabilities.

\paragraph{The Effect of Synchronizing $\bu$ and $\bw$}
\begin{figure}[t]
\centering
\begin{minipage}{0.6\linewidth}
    \centering
    \small
    \subfigure[$ \operatorname{Dist}(\boldsymbol{u}, \boldsymbol{w})$ with synchronization.]{
    \label{fig:dist_uw}
    \includegraphics[width=0.98\textwidth]{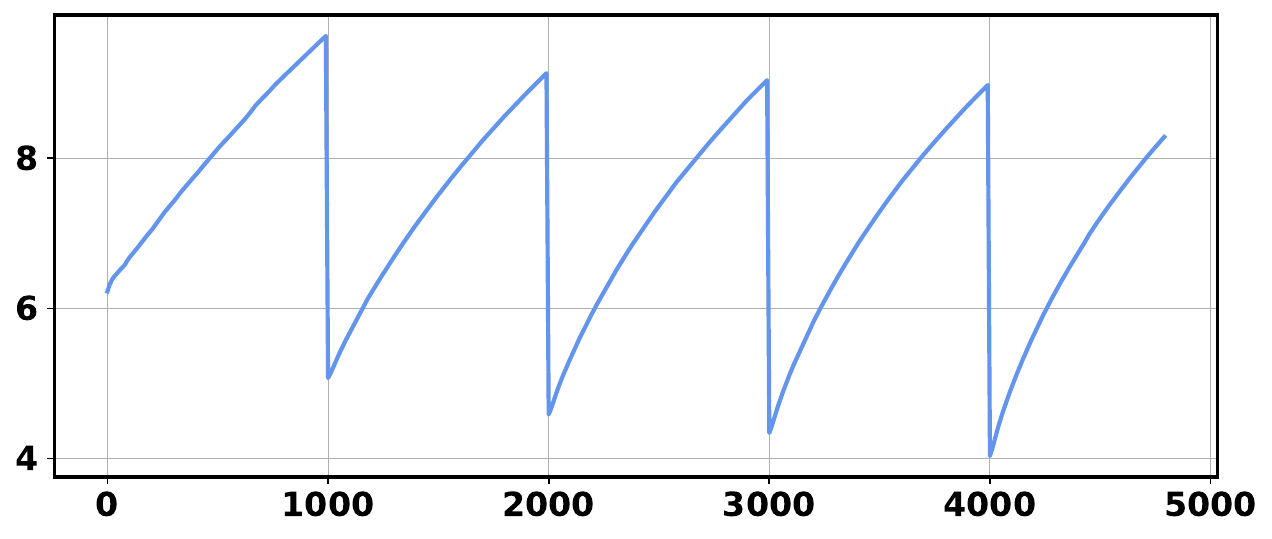}}
\end{minipage}
\begin{minipage}{0.36\linewidth}
	\centering
        \small
 	\subfigure[$\operatorname{Dist}(\boldsymbol{u}, \boldsymbol{w})$  w/o synchronization.]{	
     \label{fig:dist_uw_rho1}
    \includegraphics[width=0.9\linewidth]{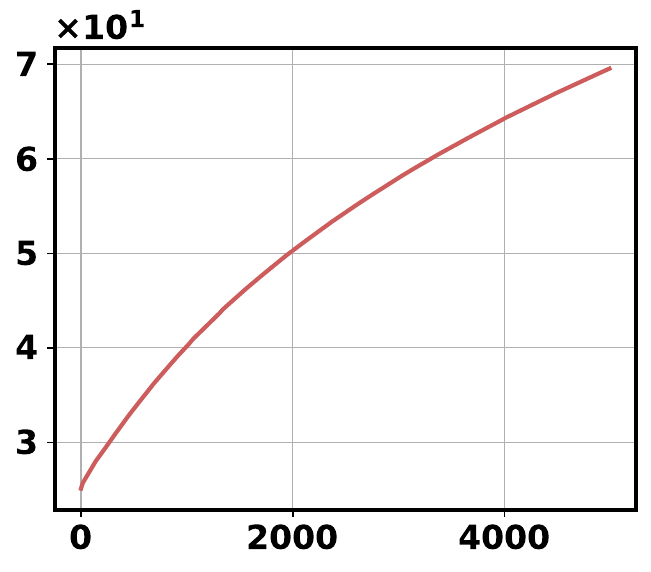}}
\label{fig:dist}
\end{minipage}
\caption{Trajectory of $\operatorname{Dist}(\bu, \bw)$. BLADE maintains the structural constraint of the bi-level objective (Equation~\eqref{eq:bi-level-penalty}) through periodic synchronization, whereas RHO-1 exhibits inevitable divergence.}
\label{fig:synchronizing_uw}
\end{figure}

One obvious characteristic of BLADE is the periodic synchronization of the proxy model $\bu$ and the reference model $\bw$ ($\bw_0^{(e)} = \bu^{(e\tau)}$). We investigate this by tracking the parameter distance $\operatorname{Dist}(\bu, \bw) = \lVert\bu - \bw\rVert_2$ during training. As shown in Figure~\ref{fig:dist_uw}, synchronization bounds the disparity between $\bu$ and $\bw$ throughout the selection process. In contrast, maintaining them independently (as in RHO-1) incurs blown-up divergence (Figure~\ref{fig:dist_uw_rho1}). Crucially, this bounded drift effectively restricts the optimization space, allowing our framework to enforce the bi-level constraint using a moderate, numerically stable penalty parameter in practice, thereby circumventing the gradient explosion typically associated with $\lambda \to \infty$.

%% file: sections/conclusion.tex
\section{Conclusion}
\label{sec:conclusion}

In this paper, we propose BLADE, a principled, efficient, and versatile data selection framework for large language model training. Starting from a bi-level formulation, BLADE reformulates the intractable bi-level objective into a penalized single-level surrogate via Lagrange multipliers, thereby eliminating inverse-Hessian computation while retaining the theoretical grounding of influence-based methods. This reformulation also reveals a clear connection to excess-loss-based selection, but crucially replaces the static reference model with a dynamic validation-informed reference model that remains synchronized with the evolving proxy model, avoiding the calibration lag that limits prior approaches. Theoretically, BLADE converges to a first-order stationary point of the bi-level objective at a rate of $\mathcal{O}(T^{-1/4})$. We further instantiate BLADE as an efficient online selection algorithm through a memoryless randomized block-coordinate Frank-Wolfe procedure, enabling practical token-level selection at LLM scale. Experiments up to 13B parameters and 10B processed tokens, along with multi-seed and checkpoint analyses, demonstrate stable gains in both domain-shift and general pre-training. Overall, BLADE provides a principled and computationally viable path for dynamic data selection in large-scale LLM training.

%% file: sections/appendix.tex
\newpage
\onecolumn

\section{Convergence of Bi-level Data Selection}
\label{sec:convergence}
 Here we present a theoretical analysis of the proposed BLADE. The original bi-level optimization problem 
\begin{equation}
\label{eq:original}
    \cP: \quad \min_{\boldsymbol{\alpha} \in \mathcal{A}} \cL_{\val}\left(\bw^{*}(\balpha)\right) \ \  \text { s.t. } \bw^{*}(\balpha)\in S^{*}(\balpha) := \arg\min_{\boldsymbol{w}} \mathcal{L}_{\train}(\balpha, \bw)
\end{equation}

is difficult to solve due to the imposed hard constraint. We turn to the Lagrangian problem of $\cP$ such that 
\begin{equation}
\label{eq:lagrange}
    \cP_{\lambda}: \quad \min_{\balpha\in\cA, \bw}\mathcal{L}_{\val}\left(\bw\right) + \lambda \left(\cL_{\train}(\balpha, \bw) - \min_{\bu}\mathcal{L}_{\train}(\balpha, \bu)\right)
\end{equation}

There has been a vast body of existing literature~\citep{shen23onpenalty,kwon2023f2sa} discussing the relationship between the Lagrangian problem $\cP_{\lambda}$ and the original problem $\cP$. 

Specifically, by employing a sufficiently large penalty $\lambda$, the solution to $\cP_{\lambda}$ closely approximates the exact solution to $\cP$. Consequently, we can solve the original problem $\cP$ by optimizing the surrogate $\mathcal{P}_\lambda$, as we did in this paper.

\subsection{The Proposed Algorithm}
\begin{algorithm}[ht]
\caption{Bi-Level Adaptive Data sElection (BLADE)}
\label{alg:workflow}
\footnotesize
\begin{algorithmic}[1]
\REQUIRE ~~\\
    Training dataset $\mathcal{D}_{\text{train}}$, 
    Validation dataset $\mathcal{D}_{\text{val}}$. \\
    Total training steps $T$, 
    Reference model update interval $\tau$, 
    Keeping ratio $\gamma \in (0,1]$. \\
    Learning rate $\eta_{\boldsymbol{u}}$ and $\eta_{\boldsymbol{w}}$ for $\boldsymbol{u}$ (proxy) and $\boldsymbol{w}$ (reference) respectively.  

\STATEx // \textit{Warm-up Phase:} \\
\STATE Initialize proxy model and train it for a few steps to obtain a stabilized initial state $\bu^{(0)}$.

\FOR{$t = 0$ \textbf{to} $T-1$}
    \STATEx // \textit{Periodic Reference Model Update.} \\
        \IF{$t \mod \tau == 0$}
            \STATE $e \leftarrow t / \tau$; \ \  $\boldsymbol{w}_{0}^{(e)} \leftarrow \boldsymbol{u}^{(t)}$.  \\
            \FOR{$k = 0$ \textbf{to} $K-1$}
                \STATE $\boldsymbol{w}^{(e)}_{k+1} = \boldsymbol{w}^{(e)}_{k} - \eta_{\boldsymbol{w}} \left( \nabla_{\boldsymbol{w}} \mathcal{L}_{\text{val}}\left(\boldsymbol{w}^{(e)}_k \right) + \lambda\nabla_{\boldsymbol{w}}\mathcal{L}_{\text{train}}\left(\boldsymbol{\alpha}^{(e \tau)}, \boldsymbol{w}^{(e)}_k \right)  \right)$. 
            \ENDFOR
            \STATE $\tilde{\bw}^{(e)} \leftarrow \boldsymbol{w}^{(e)}_{K}$
        \ENDIF
    \STATEx// \textit{Batch Data Selection.} \\
    \STATE Sample candidate batch: $\mathcal{B}^{(t)} \sim \mathcal{D}_{\text{train}}$. 
    \STATE For each selection unit $ n \in \mathcal{B}^{(t)}$, optimize $\alpha_n = I_{\text{Top}-\gamma |\mathcal{B}^{(t)}|} \left( \mathcal{L}_{\text{train}}^n(\boldsymbol{u}^{(t)}) - \mathcal{L}_{\text{train}}^n(\tilde{\bw}^{(e)}) \right) $.
    \STATEx // \textit{Online Model Update.} \\
        \STATE $\boldsymbol{u}^{(t+1)} = \boldsymbol{u}^{(t)} - \eta_{\boldsymbol{u}} \cdot \nabla_{\boldsymbol{u}} \frac{1}{\sum \alpha_n} \sum_{n \in \mathcal{B}^{(t)}} \alpha_n \mathcal{L}_{\text{train}}^n(\boldsymbol{u}^{(t)})$
\ENDFOR
\ENSURE Trained model $\boldsymbol{u}^{(T)}$
\end{algorithmic}
\end{algorithm}

A full workflow of Bi-Level Adaptive Data sElection (BLADE) is detailed in Algorithm~\ref{alg:workflow}. BLADE reformulates the complex data selection problem into a penalized single-level optimization process. Operating directly on incoming streams of candidate data, the algorithm maintains two structurally synchronized models: a proxy model $\bu$ (serving as the target LLM) and a reference model $\bw$. To evaluate data efficacy dynamically, BLADE computes the loss gap between the reference and proxy models. As the proxy model $\bu$ is trained on $\mathcal{D}_{\text{train}}$ while the reference model $\bw$ is trained on both the train and validation data, the gap between $\bu$ and $\bw$ quantifies the expected gain of incorporating a token towards the target distribution. The tokens with the  Top-$\gamma$ loss gap are selected subsequently. Notably, the proxy model $\bu$ and reference model $\bw$ are synchronized at the beginning of the periodic update of $\bw$ in BLADE.

\subsection{Convergence Rate}

Next, we establish the convergence rate of the proposed Algorithm~\ref{alg:workflow}. We first establish the following definitions to simplify the notations. We denote 
\begin{equation*}
	\cH_{\lambda}(\balpha, \bw) = \cL_{\val}(\bw) + \lambda\left(\cL_{\train}(\balpha, \bw) - \min_{\bw}\cL_{\train}(\balpha, \bw)\right), 
\end{equation*}
with $S_{\lambda}^{*}(\balpha) = \{\bw: \bw \in \arg\min_{\bw}\cH_{\lambda}(\balpha, \bw)\}$, and $S^{*}(\balpha) = \{\bw:\bw\in\arg\min_{\bw}\cL_{\train}(\balpha, \bw)\}$. We further define 
\begin{equation}
	\cH(\balpha) = \inf_{\bw\in S^{*}(\balpha)}\cL_{\val}(\bw), 
\end{equation}
The minimum of $\cH(\balpha)$ is exactly the solution to original problem $\cP$ \eqref{eq:original}. We then impose the following standard regularity conditions to derive the convergence rate:

\begin{assumption}[PL condition]\label{ass:pl condition}
	For any $\balpha\in\cA$, both $\cL_{\train}(\balpha, \bw)$ and $\cH_{\lambda}(\balpha, \bw)$ satisfy the Polyak-\L{}ojasiewicz (PL) inequality with coefficient $\mu$ and $\mu_{\lambda}$, respectively. That says:
	\begin{equation}
		\cL_{\train}(\balpha, \bw) - \min_{\bw}\cL_{\train}(\balpha, \bw) \leq \frac{1}{2\mu}\left\|\nabla_{\bw}\cL_{\train}(\balpha, \bw)\right\|^{2}
	\end{equation}
	and 
	\begin{equation}
		\cH_{\lambda}(\balpha, \bw) - \min_{\bw}\cH_{\lambda}(\balpha, \bw) \leq \frac{1}{2\mu_{\lambda}}\|\nabla\cH_{\lambda}(\balpha, \bw)\|^{2}.
	\end{equation}
	Moreover, the coefficient $\mu_{\lambda}$ satisfies $\lim_{\lambda\to\infty}\frac{\mu_{\lambda}}{\lambda} = 1$.
\end{assumption}  
\begin{assumption}[Smoothness]\label{ass:smoothness}
	For any $\balpha\in\cA$, 
	\begin{enumerate}
		\item Both $\nabla_{\bw}\cL_{\train}(\balpha, \bw)$ and $\nabla_{\bw}\cL_{\val}(\bw)$ are Lipschitz continuous to $\bw$ on coefficient $L$.
		\item Both $\cL_{\train}(\balpha, \bw)$ and $\cL_{\val}(\bw)$ are Lipschitz continuous to $\bw$ with coefficient $B$.
	\end{enumerate}
\end{assumption}
\begin{assumption}[Bounded Hessian]\label{ass:bounded hessian}
	For any $\balpha\in\cA$, $\bw\in S^{*}(\balpha)$, there exists positive constants $\kappa, \rho$, satisfying  Hessian matrices $\nabla^2_{\bw\bw}\cL_{\train}(\balpha, \bw) \succeq \kappa$\footnote{For two matrices $\bA$, $\bA\succeq \kappa$ means $\bA - \kappa \bI$ is a positively semi-definite matrix.} and $\nabla_{\balpha\bw}^{2}\cL_{\train}(\balpha, \bw)\preceq \rho$. 
\end{assumption}
\begin{assumption}[Lipschitz Hessian]\label{ass:Lipchitz hessian}
	For any $\balpha\in\cA$, $\cL_{\train}(\balpha, \bw)$ is twice continuous differentiable, and the Hessian matrices  $\nabla_{\balpha\bw}^{2}\cL_{\train}(\balpha, \bw)$, $\nabla_{\bw\bw}^{2}\cL_{\train}(\balpha, \bw)$ are all Lipschitz continuous to $\bw$ with coefficient $H$.  
\end{assumption}
\begin{assumption}[Bounded Loss Function]\label{ass:bounded optima}
		There exists a bounded region $\mathcal{W}$ containing the optimization trajectory and the local solution sets used in the analysis such that, for every $\balpha\in\mathcal{A}$ and $\bw\in\mathcal{W}$, the non-negative losses $\cL_{\train}(\balpha,\bw)$ and $\cL_{\val}(\bw)$ are bounded by a positive constant $D$.
\end{assumption}

This is a trajectory-local assumption, not a claim that cross-entropy is globally bounded: $-\log p_{\bw}(\cdot)$ is unbounded above over the full parameter space. The proof only evaluates losses on the region reached by the finite-horizon algorithm. In practice, weight decay and gradient clipping (Table~\ref{tab:hyper_params}) help keep iterates in such a stable region; a diverging loss would indicate a failed optimization run rather than the regime analyzed here.

Next, we present a lemma to characterize the gap between the gradient of $\nabla_{\balpha}\cH(\balpha)$ and $\nabla_{\balpha}\cH_{\lambda}(\balpha, \bw)$. 

\begin{lemma}[Equality of the bilevel-optimization and the regularized problem~\cite{wang2025tandem}]
\label{lem:gap between gradient}
	Under Assumptions \ref{ass:pl condition}-\ref{ass:Lipchitz hessian}, for any given $\balpha$ and $\bw_{\lambda}^{*}(\balpha)\in S_{\lambda}^{*}(\balpha)$, it holds 
	\begin{equation}
		\left\|\nabla_{\balpha}\cH(\balpha) - \frac{\partial}{\partial{\balpha}}\cH_{\lambda}(\balpha, \bw_{\lambda}^{*}(\balpha))\right\| \leq \frac{1}{\lambda}\left(\frac{HB^{2}}{\mu^{2}}\left(\frac{\rho}{\kappa} + 1\right) + \frac{\rho LB}{\mu\kappa}\right)
	\end{equation} 
\end{lemma}

From Lemma \ref{lem:gap between gradient}, we know that the gap between the gradients of original problem $\cH(\balpha)$ and its Lagrange version $\cH_{\lambda}(\balpha, \bw_{\lambda}^{*}(\balpha))$ can be extremely small by invoking penalty parameter $\lambda\to \infty$. This implies that we can directly utilize the gradient of the Lagrangian surrogate to perform gradient-based optimization. Next, we illustrate the Lipschitz continuity of $\bw_{\lambda}^{*}(\balpha)$ and $\bw^{*}(\balpha)$ to $\balpha$. 

\begin{lemma}[Continuity of $\bw^*$ and $\bw_{\lambda}^*$~\cite{wang2025tandem}]
\label{lem:continuous w}
	Under Assumptions \ref{ass:pl condition} and \ref{ass:smoothness}, for any $\balpha_{1}, \balpha_{2}\in \cA$, it holds 
	\begin{equation}
		\|\bw^{*}(\balpha_{1}) - \bw^{*}(\balpha_{2})\| \leq \frac{L}{\mu}\|\balpha_{1} - \balpha_{2}\|, 
	\end{equation}
	for any $\bw^{*}(\balpha_{1})\in S^{*}(\balpha_1)$ and $\bw^{*}(\balpha_{2})\in S^{*}(\balpha_{2})$ satisfies $\bw^{*}(\balpha_{2}) = \arg\min_{\bw\in S^{*}(\balpha_{2})}\|\bw - \bw^{*}(\balpha_{1})\|$.	On the other hand, it holds
	\begin{equation}
		\|\bw^{*}_{\lambda}(\balpha_{1}) - \bw_{\lambda}^{*}(\balpha_{2})\| \leq \frac{\lambda L}{\mu_{\lambda}}\|\balpha_{1} - \balpha_{2}\| 
	\end{equation}
	for any $\bw^{*}_{\lambda}(\balpha_{1})\in S_{\lambda}^{*}(\balpha_{1})$ and $\bw_{\lambda}^{*}(\balpha_{2}) = \arg\min_{\bw\in S_{\lambda}^{*}(\balpha_{1})}\|\bw - \bw^{*}_{\lambda}(\balpha_{1})\|$. 
\end{lemma}
From Lemma \ref{lem:continuous w}, we see that $\bw_{\lambda}^{*}(\balpha)$ and $\bw^{*}(\balpha)$ are Lipchitz continuous with coefficient $L_{\bw^*} = \frac{L}{\mu}$ and $L_{\bw^*_{\lambda}} = \frac{\lambda L}{\mu_\lambda}$

\begin{lemma}[Gradient Smoothness of $\cH_{\lambda}(\balpha, \bw_{\lambda}^{*}(\balpha))$~\cite{wang2025tandem}]
\label{lem:gradient_smoothness}
	Under Assumptions \ref{ass:pl condition} and \ref{ass:smoothness}, $\nabla_{\balpha}\cH_{\lambda}(\balpha, \bw_{\lambda}^{*}(\balpha))$ has semi-Lipschitz gradient such that 
	\begin{equation}
		\|\nabla_{\balpha}\cH_{\lambda}(\balpha_{1}, \bw_{\lambda}^{*}(\balpha_{1})) - \nabla_{\balpha}\cH_{\lambda}(\balpha_{2}, \bw_{\lambda}^{*}(\balpha_{2}))\| \leq \underbrace{\lambda B\left(\frac{\lambda L}{\mu_{\lambda}} + \frac{L}{\mu}\right)}_{L_{\lambda}}\|\balpha_{1} - \balpha_{2}\|. 
	\end{equation} 
\end{lemma}
From this lemma, we can obtain the gradient smoothness of the penalized problem $\nabla_{\balpha}\cH_{\lambda}(\balpha, \bw_{\lambda}^{*}(\balpha))$ w.r.t. $\balpha$ in $L_{\lambda} = \lambda B \left(\frac{\lambda L}{\mu_{\lambda}} + \frac{L}{\mu} \right)$.

Following the Frank-Wolfe update in BLADE, we utilize the Frank-Wolfe gap to characterize the stationary condition. 
\begin{definition}[Frank-Wolfe Gap]
    The first-order stationarity measure for constrained non-convex optimization is the Frank-Wolfe Gap:
    \begin{equation}
        G(\balpha^{(t)}) = \max_{\bs \in \mathcal{A}} \langle \nabla_{\balpha}  \mathcal{H}(\balpha), \balpha^{(t)} - \bs \rangle
    \end{equation}
    $G(\balpha^{(t)}) \ge 0$ always holds, and $G(\balpha^{(t)}) = 0$ if and only if $\balpha^{(t)}$ is a first-order stationary point.
\end{definition}

Similarly, we denote the Frank-Wolfe Gap for the penalized optimization problem as $G_{\lambda}(\balpha^{(t)}) = \max_{\bs \in \mathcal{A}} \langle \nabla_{\balpha}  \mathcal{H}_{\lambda}(\balpha, \bw_{\lambda}^*(\balpha)), \balpha^{(t)} - \bs \rangle$.

\setcounter{theorem}{0}
\begin{theorem}
	Under Assumptions \ref{ass:pl condition}-\ref{ass:bounded optima}, for sufficiently large $T$, the $\balpha^{(t)}$ obtained in Algorithm \ref{alg:workflow} holds 
	\begin{equation}
		\min_{1 \le t \le T} G(\balpha^{(t)}) \le \frac{1}{T} \sum_{t=1}^T G(\balpha^{(t)}) \le \mathcal{O}\left(T^{-\frac{1}{4}}\right)
	\end{equation}
	by selecting $\lambda = T^{\frac{1}{4}}$, $\eta_{\balpha} = T^{-\frac{3}{4}}$, $\eta_{\bu} = \frac{1}{L}$, $\eta_{\bw} = \frac{1}{L_{\lambda}}$  $K \geq  \frac{3 L_\lambda}{2 \mu_\lambda} \log (T)$. $G(\balpha^{(t)})$ is the Frank-Wolfe gap of the bi-level optimization problem $G(\balpha^{(t)}) = \max_{\bs \in \mathcal{A}} \langle \balpha^{(t)} - \bs, \nabla_{\balpha} \mathcal{H}(\balpha^{(t)}) \rangle$.
\end{theorem}

\begin{proof} 
    At step $t$, the algorithm~\ref{alg:workflow} estimates the hyper-gradient $\Delta^{(t)}$ using the delayed reference model $\tilde{\bw}^{(e)}$ and the online proxy model $\bu^{(t)}$:
    \begin{equation*}
        \Delta^{(t)} = \lambda \left( \nabla_{\balpha} \mathcal{L}_{\text{train}}(\balpha^{(t)}, \tilde{\bw}^{(e)}) - \nabla_{\balpha} \mathcal{L}_{\text{train}}(\balpha^{(t)}, \bu^{(t)}) \right)
    \end{equation*}
    Subsequently, the linear minimization oracle (LMO) solves for the optimal descent direction $\bs^{(t)} \in \mathcal{A}$ and updates $\balpha$:
    \begin{equation*}
        \bs^{(t)} = \arg\min_{\bs \in \mathcal{A}} \langle \Delta_t, \bs \rangle, \quad \balpha^{(t+1)} = \balpha^{(t)} + \eta_{\balpha} (\bs^{(t)} - \balpha^{(t)})
    \end{equation*}

    By Lemma \ref{lem:gradient_smoothness}, $ \mathcal{H}_{\lambda}(\balpha, \bw_{\lambda}^*(\balpha))$ has $L_\lambda$-Lipschitz continuous gradients. Applying the standard descent inequality yields:
    \begin{equation} 
    \begin{split}
        \label{eq:descent}
        \mathcal{H}_{\lambda}(\balpha^{(t+1)}, \bw_{\lambda}^*(\balpha^{(t+1)})) \le & \mathcal{H}_{\lambda}(\balpha^{(t)}, \bw_{\lambda}^*(\balpha^{(t)})) + \eta_{\balpha} \langle \nabla_{\balpha}  \mathcal{H}_{\lambda}(\balpha^{(t)}, \bw_{\lambda}^*(\balpha^{(t)})), \bs^{(t)} - \balpha^{(t)} \rangle \\ & + \frac{L_\lambda}{2} \eta_{\balpha}^2 \|\bs^{(t)} - \balpha^{(t)}\|^2
    \end{split}
    \end{equation}
    Let $\bar{\bs}^{(t)} = \arg\min_{\bs \in \mathcal{A}} \langle \nabla_{\balpha}  \mathcal{H}_{\lambda}(\balpha^{(t)}, \bw_{\lambda}^*(\balpha^{(t)})), \bs \rangle$ be the ideal LMO direction. By definition of $\bs^{(t)}$ minimizing the approximate gradient $\Delta^{(t)}$, we introduce the gradient tracking error $\delta^{(t)} = \|\Delta^{(t)} - \nabla_{\balpha}  \mathcal{H}_{\lambda}(\balpha^{(t)}, \bw_{\lambda}^*(\balpha^{(t)}))\|$. Using the Cauchy-Schwarz inequality and domain diameter bound $\|\bs^{(t)} - \bar{\bs}^{(t)}\| \le D_{\mathcal{A}} = \sqrt{N}$:
    \begin{align*}
        \langle \nabla_{\balpha} \mathcal{H}_{\lambda}(\balpha^{(t)}, \bw_{\lambda}^*(\balpha^{(t)})), \bs^{(t)} \rangle &= \langle \Delta^{(t)}, \bs^{(t)} \rangle + \langle \nabla_{\balpha} \mathcal{H}_{\lambda}(\balpha^{(t)}, \bw_{\lambda}^*(\balpha^{(t)})) - \Delta^{(t)}, \bs^{(t)} \rangle \\
        &\le \langle \Delta^{(t)}, \bar{\bs}^{(t)} \rangle + \langle \nabla_{\balpha} \mathcal{H}_{\lambda}(\balpha^{(t)}, \bw_{\lambda}^*(\balpha^{(t)})) - \Delta^{(t)}, \bs^{(t)} \rangle \\
        &= \langle \nabla_{\balpha} \mathcal{H}_{\lambda}(\balpha^{(t)}, \bw_{\lambda}^*(\balpha^{(t)})), \bar{\bs}^{(t)} \rangle \\
        & + \langle \nabla_{\balpha} \mathcal{H}_{\lambda}(\balpha^{(t)}, \bw_{\lambda}^*(\balpha^{(t)})) - \Delta^{(t)}, \bs^{(t)} - \bar{\bs}^{(t)} \rangle \\
        &\le \langle \nabla_{\balpha} \mathcal{H}_{\lambda}(\balpha^{(t)}, \bw_{\lambda}^*(\balpha^{(t)})), \bar{\bs}^{(t)} \rangle + \delta^{(t)} D_{\mathcal{A}}
    \end{align*}
    Substituting this back into \Cref{eq:descent} and utilizing $\langle \nabla_{\balpha} \mathcal{H}_{\lambda}(\balpha^{(t)}, \bw_{\lambda}^*(\balpha^{(t)})), \bar{\bs}^{(t)} - \balpha^{(t)} \rangle = -G_{\lambda}(\balpha^{(t)})$:
    \begin{equation}
    \label{eq:inexact_fw_descent}
         \mathcal{H}_{\lambda}(\balpha^{(t+1)}, \bw_{\lambda}^*(\balpha^{(t+1)})) \le  \mathcal{H}_{\lambda}(\balpha^{(t)}, \bw_{\lambda}^*(\balpha^{(t)})) - \eta_{\balpha} G_{\lambda}(\balpha^{(t)}) + \eta_{\balpha} \delta^{(t)} D_{\mathcal{A}} + \frac{L_\lambda}{2} \eta_{\balpha}^2 D_{\mathcal{A}}^2
    \end{equation}

    By the envelope theorem~\footnote{While the PL condition (Assumption 1) guarantees that every stationary point is a global minimum, it does not imply a unique minimizer. To strictly justify the application of the envelope theorem in deriving $\nabla_{\alpha}\mathcal{H}_{\lambda}(\alpha, w_{\lambda}^{*}(\alpha))$, we rely on the well-established implicit bias of gradient descent. Since the inner proxy model $\bu^{(t)}$ is updated via SGD from a fixed initialization, the optimization trajectory deterministically tracks a specific minimum-norm solution along the solution manifold. This implicit regularization ensures that, locally, the mapped target $\bw^*(\balpha)$ is unique and continuous with respect to $\balpha$, making the envelope theorem formally applicable in our dynamic tracking context~\cite{neyshabur14implicitreg,soudry18implicitbias}.}, $ \nabla_{\alpha}  \mathcal{H}_{\lambda}(\balpha, \bw_{\lambda}^*(\balpha)) = \lambda (\nabla_{\balpha} \mathcal{L}_{\text{train}}(\balpha, \bw_\lambda^*(\alpha)) - \nabla_{\balpha} \mathcal{L}_{\text{train}}(\balpha, \bw^*(\balpha)))$. Together with the Lipschitz smoothness of $\nabla_{\balpha} \mathcal{L}_{\text{train}}$ w.r.t. parameters, we have:
    \begin{equation}
    \label{eq:tracking_error}
        \delta^{(t)} \le \lambda L \left( \|\tilde{\bw}^{(e)} - \bw_\lambda^*(\balpha^{(t)})\| + \|\bu^{(t)} - \bw^*(\balpha^{(t)})\| \right)
    \end{equation}
    We bound the two parameter distances:
    \paragraph{Reference Model Drift:} The delayed reference model $\tilde{\bw}^{(e)}$ is optimized for $K$ steps at $t=e\tau$ and keep fixed for $\tau$ steps. The error is the sum of the optimization residual and target drift over $\tau$ steps:
    \begin{align*}
        \|\tilde{\bw}^{(e)} - \bw_\lambda^*(\balpha^{(t)})\| &\le \|\tilde{\bw}^{(e)} - \bw_\lambda^*(\balpha^{(e\tau)})\| + \|\bw_\lambda^*(\balpha^{(e\tau)}) - \bw_\lambda^*(\balpha^{(t)})\| \\
        & \le \sqrt{\frac{2D(1+\lambda)}{\mu_\lambda}} (1 - \frac{\mu_\lambda}{L_{\lambda}})^{K/2} + \frac{\lambda L}{\mu_\lambda} \tau \eta_\alpha D_\mathcal{A} \\
        & = \mathcal{O}(\tau \eta_{\balpha})
    \end{align*}

    \paragraph{Proxy Model Tracking:} The proxy model updates online, tracking the moving target $\bw^*(\balpha^{(t)})$. Denoting the proxy optimization residual $\mathcal{R}(\balpha^{(t)}, \bu^{(t)}) = \mathcal{L}_{\text{train}}(\balpha^{(t)}, \bu^{(t)}) - \mathcal{L}_{\text{train}}^*(\balpha^{(t)})$, by the Quadratic Growth of $\mathcal{L}_{\text{train}}$, we have:

    \begin{equation}
         \|\bu^{(t)} - \bw^*(\balpha^{(t)})\| \le \sqrt{\frac{2}{\mu} \mathcal{R}(\balpha^{(t)}, \bu^{(t)})}
    \end{equation}
    Bringing these two into Equation~\eqref{eq:tracking_error}, we have
    \begin{equation*}
        \delta^{(t)} \le \lambda L \sqrt{\frac{2}{\mu}} \sqrt{\mathcal{R}(\balpha^{(t)}, \bu^{(t)})} + \mathcal{O}(\lambda \tau \eta_{\balpha})
    \end{equation*}
    Substituting this $\delta^{(t)}$ into \eqref{eq:inexact_fw_descent} introduces a cross term $\left( \eta_{\balpha} D_{\mathcal{A}} \lambda L \sqrt{\frac{2}{\mu}} \right) \sqrt{\mathcal{R}(\balpha^{(t)}, \bu^{(t)})}$. To decouple this term, we apply the AM-GM inequality ($x y \le \frac{\beta}{2} x^2 + \frac{1}{2\beta} y^2$) with coefficient $\beta = c \eta_{\bu} \mu$ (where $c>0$ is a constant):
    \begin{align*}
        \left( \eta_{\balpha} D_{\mathcal{A}} \lambda L \sqrt{\frac{2}{\mu}} \right) \sqrt{\mathcal{R}(\balpha^{(t)}, \bu^{(t)})} \le \frac{c \eta_{\bu} \mu}{2} \mathcal{R}(\balpha^{(t)}, \bu^{(t)}) + \frac{\left( \eta_\alpha D_{\mathcal{A}} \lambda L \sqrt{2/\mu} \right)^2}{2 c \eta_{\bu} \mu}
    \end{align*}
    Note that the residual error term is bounded by $\mathcal{O}(\lambda^2 \eta_{\balpha}^2)$. Thus, the outer descent inequality~\eqref{eq:inexact_fw_descent} becomes:
    \begin{equation}
         \mathcal{H}_{\lambda}(\balpha^{(t+1)}, \bw_{\lambda}^*(\balpha^{(t+1)})) \le  \mathcal{H}_{\lambda}(\balpha^{(t)}, \bw_{\lambda}^*(\balpha^{(t)})) - \eta_{\balpha} G_{\lambda}(\balpha^{(t)}) + \frac{c \eta_{\bu} \mu}{2} \mathcal{R}(\balpha^{(t)}, \bu^{(t)}) + \mathcal{O}(\lambda^2 \eta_{\balpha}^2) 
    \end{equation}
    
    For the inner proxy model $\bu$, one step of SGD yields a strict geometric contraction of the residual under the PL condition. Subsequently, the outer data weight $\alpha$ updates, causing the optimal target to drift. By utilizing the smoothness of $\mathcal{L}_{\text{train}}$ and applying AM-GM to absorb the linear drift into the residual, the inner tracking dynamics can be strictly bounded as:
    \begin{align}
        \mathcal{R}(\balpha^{(t+1)}, \bu^{(t+1)}) \le (1 - \eta_{\bu} \mu) \mathcal{R}(\balpha^{(t)}, \bu^{(t)}) + \mathcal{O}(\eta_\alpha^2) \label{eq:inner_tracking}
    \end{align}

    We then define the global Lyapunov potential function as the weighted sum of the outer penalized objective and the inner optimization residual:
    \begin{align*}
        V^{(t)} := \mathcal{H}_{\lambda}(\balpha^{(t)}, \bw_{\lambda}^*(\balpha^{(t)})) + c \cdot \mathcal{R}(\balpha^{(t)}, \bu^{(t)})
    \end{align*}
    To compute the difference $V^{(t+1)} - V^{(t)}$, we add Equation~\eqref{eq:inexact_fw_descent} and $c \times$ Equation~\eqref{eq:inner_tracking}:
    \begin{align}
        V^{(t+1)} - V^{(t)} &\le - \eta_{\balpha} G_{\lambda}(\balpha^{(t)}) + \left( \frac{c \eta_{\bu} \mu}{2} - c \eta_{\bu} \mu \right) \mathcal{R}(\balpha^{(t)}, \bu^{(t)}) + \mathcal{O}(\lambda^2 \eta_{\balpha}^2) \nonumber \\
    &= - \eta_{\balpha} G_{\lambda}(\balpha^{(t)}) - \frac{c \eta_{\bu} \mu}{2} \mathcal{R}(\balpha^{(t)}, \bu^{(t)}) + \mathcal{O}(\lambda^2 \eta_{\balpha}^2) \label{eq:lyapunov_diff}
    \end{align}
    Crucially, the coefficient for $\mathcal{R}(\balpha^{(t)}, \bu^{(t)})$ is $- \frac{c \eta_{\bu} \mu}{2} \le 0$. This indicates that under the two-timescale framework, the descent energy provided by the inner SGD perfectly cancels out the gradient errors injected into the outer optimization via AM-GM. Dropping the non-positive term, we obtain a clean global descent bound:
    \begin{align*}
        V^{(t+1)} \le V^{(t)} - \eta_{\balpha} G_{\lambda}(\balpha^{(t)}) + \mathcal{O}(\lambda^2 \eta_{\balpha}^2)
    \end{align*}
    Summing the above inequality from $t=1$ to $T$ and utilizing $V^{(T)} \ge \min_{\balpha \in \mathcal{A}} \mathcal{H}_{\lambda}(\balpha, \bw_{\lambda}^*(\balpha) \ge 0$, we have:
    \begin{align}
    \label{eq:Lyapunov_bound}
        \frac{1}{T} \sum_{t=1}^T G_{\lambda}(\balpha^{(t)}) & \le \frac{V^{(0)}}{T \eta_{\balpha}} + \mathcal{O}(\lambda^2 \eta_{\balpha}) 
    \end{align}

    Now let's bound the Frank-Wolfe Gap of the original bi-level optimization problem. By utilizing the subadditivity of the max operator~($|\max f(x) - \max g(x)| \le \max |f(x) - g(x)|$)
    and the Cauchy-Schwarz inequality, we have  
    \begin{align*}
        |G(\boldsymbol{\alpha}^{(t)}) - G_\lambda(\boldsymbol{\alpha}^{(t)})| & \le \max_{\boldsymbol{s} \in \mathcal{A}} \left| \langle \boldsymbol{\alpha}^{(t)} - \boldsymbol{s}, \nabla_{\balpha} \mathcal{H}(\balpha^{(t)}) - \nabla_{\balpha} \mathcal{H}_{\lambda}(\balpha^{(t)}, \bw_{\lambda}^*(\balpha^{(t)}))  \rangle \right| \\
        & \le \max_{\boldsymbol{s} \in \mathcal{A}} \left( \|\boldsymbol{\alpha}^{(t)} - \boldsymbol{s}\| \cdot \|\nabla_{\balpha} \mathcal{H}(\balpha^{(t)}) - \nabla_{\balpha} \mathcal{H}_\lambda(\boldsymbol{\alpha}^{(t)},  \bw_{\lambda}^*(\balpha^{(t)})\| \right) \\
        & \le D_{\mathcal{A}} \cdot \|\nabla_{\balpha} \mathcal{H}(\balpha^{(t)}) - \nabla_{\balpha} \mathcal{H}_\lambda(\boldsymbol{\alpha}^{(t)},  \bw_{\lambda}^*(\balpha^{(t)})\|
    \end{align*}
    This gives: 
    \begin{equation}
    \begin{split}
    \label{eq:true_penalized_contraction}
        G(\boldsymbol{\alpha}^{(t)}) & \le G_{\lambda}(\boldsymbol{\alpha}^{(t)}) +   D_{\mathcal{A}} \cdot \|\nabla_{\balpha} \mathcal{H}(\balpha^{(t)}) - \nabla_{\balpha} \mathcal{H}_\lambda(\boldsymbol{\alpha}^{(t)},  \bw_{\lambda}^*(\balpha^{(t)})\| \\
    \end{split}
    \end{equation}
    Based on the approximation guarantee of the penalty method (Lemma~\ref{lem:gap between gradient}), the gradient divergence is bounded by $\mathcal{O}(1/\lambda)$. Taking the average over $T$ iterations, we obtain:
    \begin{equation}
        \frac{1}{T} \sum_{t=1}^T G(\boldsymbol{\alpha}^{(t)}) \le \frac{1}{T} \sum_{t=1}^T G_\lambda(\boldsymbol{\alpha}^{(t)}) + \mathcal{O}\left(\frac{1}{\lambda}\right)
    \end{equation}
    Substituting the result from Equation~\eqref{eq:Lyapunov_bound}, we complete theoretical bound:
    \begin{equation*}
        \frac{1}{T} \sum_{t=1}^T G(\boldsymbol{\alpha}^{(t)}) \le \mathcal{O}\left( \frac{1}{T \eta_\alpha} \right) + \mathcal{O}(\lambda^2 \eta_\alpha) + \mathcal{O}\left( \frac{1}{\lambda} \right)
    \end{equation*}
    
    By setting the parameters $\lambda = T^{\frac{1}{4}}$ and $\eta_{\balpha} = T^{-\frac{3}{4}}$, all three terms are perfectly balanced, yielding the global convergence rate for the original bi-level problem:
    \begin{equation*}
        \min_{1 \le t \le T} G(\boldsymbol{\alpha}^{(t)}) \le \frac{1}{T} \sum_{t=1}^T G(\boldsymbol{\alpha}^{(t)}) \le \mathcal{O}(T^{-\frac{1}{4}})
    \end{equation*}
\end{proof}

\section{Derivation of the Hyper-gradient}
\label{sec:hyper-gradient}
To maintain self-containment and rigorously justify the connection discussed in Section~\ref{subsec:unifies}, we provide the derivation of the hyper-gradient for our specific data selection formulation using the standard Implicit Function Theorem.

The bilevel optimization problem for data weighting is defined as:
\begin{align*}
    \min_{\balpha \in \mathcal{A}} \cL_{\val}(\bw(\balpha))&: = \mathcal{L}_{\val}\left(\bw(\balpha)\right) \\
    \text { s.t. } \bw(\balpha)\in\arg \min_{\boldsymbol{w}} \cL_{\train}(\balpha, \bw) &:=  \frac{1}{\gamma N}\sum_{n=1}^N \alpha_n\mathcal{L}_{\text{train}}^{n}(\boldsymbol{w}).
\end{align*}

To optimize this problem, we compute the hyper-gradient using the chain rule:
\begin{equation}
    \nabla_{\boldsymbol{\alpha}} \mathcal{L}_{\text{val}}(\boldsymbol{w}(\boldsymbol{\alpha})) = \left( \frac{\partial \boldsymbol{w}(\boldsymbol{\alpha})}{\partial \boldsymbol{\alpha}} \right)^T \nabla_{\boldsymbol{w}} \mathcal{L}_{\text{val}}(\boldsymbol{w}) 
    \label{eq:chain_rule}
\end{equation}

Assuming the inner optimization reaches a local minimum, the first-order optimality condition holds:
\begin{equation*}
    \nabla_{\boldsymbol{w}} \mathcal{L}_{\text{train}}(\boldsymbol{\alpha}, \boldsymbol{w}(\boldsymbol{\alpha})) = \mathbf{0}
\end{equation*}
Taking the total derivative with respect to $\boldsymbol{\alpha}$, we get
\begin{equation*}
    \nabla_{\boldsymbol{\alpha}\boldsymbol{w}}^2 \mathcal{L}_{\text{train}}(\boldsymbol{\alpha}, \boldsymbol{w}) + \nabla_{\boldsymbol{w}\boldsymbol{w}}^2 \mathcal{L}_{\text{train}}(\boldsymbol{\alpha}, \boldsymbol{w}) \frac{\partial \boldsymbol{w}(\boldsymbol{\alpha})}{\partial \boldsymbol{\alpha}} = \mathbf{0}
\end{equation*}
Letting $\boldsymbol{H} = \nabla_{\boldsymbol{w}\boldsymbol{w}}^2 \mathcal{L}_{\text{train}}(\boldsymbol{\alpha}, \boldsymbol{w})$ denote the Hessian matrix, we get the Jacobian:
\begin{equation}
    \frac{\partial \boldsymbol{w}(\boldsymbol{\alpha})}{\partial \boldsymbol{\alpha}} = - \boldsymbol{H}^{-1} \nabla_{\boldsymbol{\alpha}\boldsymbol{w}}^2 \mathcal{L}_{\text{train}}(\boldsymbol{\alpha}, \boldsymbol{w})
    \label{eq:jacobian}
\end{equation}

Based on the inner objective, the gradient w.r.t. the model weights is:
\begin{equation*}
    \nabla_{\boldsymbol{w}} \mathcal{L}_{\text{train}}(\boldsymbol{\alpha}, \boldsymbol{w}) = \frac{1}{\gamma N} \sum_{n=1}^N \alpha_n \nabla_{\boldsymbol{w}} \mathcal{L}_{\text{train}}^n(\boldsymbol{w})
\end{equation*}
Taking the derivative with respect to the hyper-parameters $\boldsymbol{\alpha}$, we get:
\begin{equation*}
    \nabla_{\boldsymbol{\alpha}\boldsymbol{w}}^2 \mathcal{L}_{\text{train}}(\boldsymbol{\alpha}, \boldsymbol{w}^*) = \frac{1}{\gamma N} \boldsymbol{J}_{\text{train}}
\end{equation*}
where $\boldsymbol{J}_{\text{train}} = \left[ \nabla_{\boldsymbol{w}} \mathcal{L}_{\text{train}}^1(\boldsymbol{w}), \dots, \nabla_{\boldsymbol{w}} \mathcal{L}_{\text{train}}^N(\boldsymbol{w}) \right]$ is the matrix of individual training sample gradients.
Substituting Equation \eqref{eq:jacobian} into Equation \eqref{eq:chain_rule}, we obtain the final hyper-gradient:
\begin{equation*}
    \nabla_{\boldsymbol{\alpha}} \mathcal{L}_{\text{val}}(\boldsymbol{w}(\boldsymbol{\alpha})) = - \frac{1}{\gamma N} \boldsymbol{J}_{\text{train}}^T \boldsymbol{H}^{-1} \nabla_{\boldsymbol{w}} \mathcal{L}_{\text{val}}(\boldsymbol{w})
\end{equation*}
Thus, for a particular entry, $ \frac{\partial}{\partial\alpha_n} \mathcal{L}_{\text{val}}(\boldsymbol{w}(\boldsymbol{\alpha})) \propto - \nabla_{\bw} \mathcal{L}_{\text{train}}^n(\boldsymbol{w})^T \boldsymbol{H}^{-1} \nabla_{\boldsymbol{w}} \mathcal{L}_{\text{val}}(\boldsymbol{w})$.

\section{Hyper-parameters}
\label{sec:hyper_params}

\begin{table}[]
\caption{Hyper-parameters of BLADE for different application scenarios}
\label{tab:hyper_params}
\centering
\begin{tabular}{@{}lccc@{}}
\toprule
                                 & \multicolumn{2}{c}{Domain Shift Pretraining} & General Pretraining \\ \cmidrule(l){2-4} 
                                 & TinyLlama-1.1B          & Llama2-7B          & TinyLlama-1.1B      \\ \midrule
Batch Size                       & 512                     & 512                & 512                 \\
Learning Rate $\eta_{\bu}$       & 8e-5                    & 2e-5               & 8e-5                 \\
Learning Rate $\eta_{\bw}$       & 8e-5                    & 2e-5               & 8e-5                 \\
Learning Rate Scheduler          & Constant                & Constant           & Constant        \\
Penalty $\lambda$                & 1                    & 1               & 1                   \\
Reference Update Steps $K$       & 300                     & 300                & 300                    \\
Reference Update Interval $\tau$ & 1000                    & 1000               & 1000                    \\
Total Steps (w.r.t $\bu$)        & 5000                    & 5000               & 5000                \\
Context Length                   & 2048                    & 2048               & 2048                \\
Weight Decay                     & 0.01                    & 0.01               & 0.01                    \\
Gradient Clipping                & 1.0                     & 1.0                & 1.0                     \\ \bottomrule
\end{tabular}
\end{table}

\paragraph{Hyper-parameter settings} The detailed hyper-parameters for the BLADE algorithms in the domain-shift continuous pre-training and general continious pre-training scenario are shown in Table~\ref{tab:hyper_params}. 

While Theorem~\ref{thm:convergence} and Lemma~\ref{lem:gap between gradient} establish that the penalized formulation recovers the original bi-level objective as $\lambda \to \infty$, a very large penalty causes gradient explosion and numerical instability in practical neural-network training. BLADE closes this theory--practice gap through periodic synchronization ($\bw_0^{(e)} = \bu^{(e\tau)}$). Because the reference and proxy are repeatedly realigned, their drift, and hence the inner-optimality residual that a large $\lambda$ is intended to suppress, remains bounded (Figure~\ref{fig:dist_uw}). We therefore use the moderate value $\lambda=1$ in all experiments: the constraint is controlled structurally, while the update avoids the instability of an excessively large penalty.

\section{Additional Experiments}
\label{sec:additional_exps}

\subsection{Larger-token General Pre-training}
\label{sec:larger_tokenexps}
To examine whether BLADE's advantage persists with a larger data budget, we extend the general-pretraining candidate pool from 5B to 10B tokens. Base-CT trains on all 10B tokens, whereas RHO-1 and BLADE retain $\gamma=0.6$ of the stream and therefore perform parameter updates on 6B tokens.
BLADE achieves 52.37 average accuracy, exceeding RHO-1 by 0.99 points and Base-CT by 2.46 points. The improvement at the larger candidate budget complements the model-scaling results in Table~\ref{tab:sota_math} and shows that BLADE remains effective when more candidate data are processed.

\begin{table}[h]
\caption{General pre-training with a 10B-token candidate budget. Selection methods retain $\gamma=0.6$ of the processed tokens. All entries are accuracy (\%).}
\label{tab:scale_general_10b}
\centering
\setlength\tabcolsep{2.2pt}
\scriptsize
\begin{tabular}{@{}lc|ccccccccc@{}}
\toprule
Method & Train Toks. & MMLU & HellaSwag & OBQA & WinoGrande & ARC-C & ARC-E & BoolQ & PIQA & Avg. \\ \midrule
Base & -- & 25.31 & 59.20 & \textbf{36.00} & 59.12 & 30.12 & 55.25 & 57.83 & 73.29 & 49.52 \\
Base-CT & 10B & 26.23 & 57.51 & 35.30 & 58.98 & 30.13 & 56.29 & 62.03 & 72.84 & 49.91 \\
RHO-1$\dagger$ & 6B & 26.39 & 59.15 & 34.00 & \textbf{61.20} & 32.68 & 56.70 & 68.71 & 72.21 & 51.38 \\
BLADE & 6B & \textbf{27.95} & \textbf{60.78} & 34.20 & 60.98 & \textbf{33.02} & \textbf{59.10} & \textbf{69.60} & \textbf{73.34} & \textbf{52.37} \\ \bottomrule
\end{tabular}
\end{table}

\subsection{Multi-seed General Pre-training with Standard Deviations}
\label{sec:general_multiseed}

To quantify variability in the general-pretraining setting, we repeat every method three times at the original 5B-token candidate budget and report the mean and standard deviation for each downstream task. All selection methods use the same keeping ratio $\gamma=0.6$ and thus train on 3B retained tokens per run.

\begin{table}[h]
\caption{Three seed general-pretraining results at the original 5B token budget (mean $\pm$ standard deviation). Selection methods train on 3B retained tokens.}
\label{tab:general_multiseed}
\centering
\setlength\tabcolsep{1.6pt}
\scriptsize
\resizebox{\textwidth}{!}{%
\begin{tabular}{@{}l|ccccccccc@{}}
\toprule
Method & MMLU & HellaSwag & OBQA & WinoGrande & ARC-C & ARC-E & BoolQ & PIQA & Avg. \\ \midrule
Base & $25.18{\pm}0.09$ & $59.23{\pm}0.02$ & $35.47{\pm}0.38$ & $59.07{\pm}0.04$ & $30.86{\pm}0.52$ & $55.68{\pm}0.30$ & $57.81{\pm}0.01$ & $73.18{\pm}0.08$ & $49.56{\pm}0.03$ \\
Base-CT & $25.62{\pm}0.22$ & $58.28{\pm}0.37$ & $34.33{\pm}0.47$ & $58.72{\pm}1.29$ & $30.86{\pm}0.44$ & $55.85{\pm}0.31$ & $61.91{\pm}1.31$ & $72.10{\pm}0.45$ & $49.71{\pm}0.13$ \\
Random &  $25.31{\pm}1.50$   &  $58.43{\pm}0.52$  &   $34.67{\pm}0.81$    &   $59.84{\pm}0.20$   &   $31.07{\pm}0.96$   &   $55.49{\pm}0.19$    &   $60.24{\pm}1.21$   &   $71.78{\pm}0.46$    &  $49.60{\pm}0.17$    \\
MATES$^*$ & $25.20{\pm}0.43$ & $58.48{\pm}0.04$ & $35.27{\pm}0.66$ & $60.16{\pm}0.16$ & $31.57{\pm}0.79$ & $57.48{\pm}0.16$ & $63.41{\pm}1.11$ & $72.32{\pm}0.38$ & $50.49{\pm}0.08$ \\
RHO-1$\dagger$ & $25.99{\pm}0.35$ & $59.47{\pm}0.50$ & $\mathbf{35.80{\pm}0.59}$ & $60.43{\pm}0.26$ & $31.86{\pm}0.44$ & $56.11{\pm}0.83$ & $63.20{\pm}1.08$ & $72.32{\pm}0.62$ & $50.65{\pm}0.26$ \\
BLADE & $\mathbf{27.95{\pm}0.60}$ & $\mathbf{60.28{\pm}0.51}$ & $34.93{\pm}0.09$ & $\mathbf{60.59{\pm}0.55}$ & $\mathbf{32.82{\pm}0.45}$ & $\mathbf{58.73{\pm}0.26}$ & $\mathbf{66.34{\pm}1.52}$ & $\mathbf{73.34{\pm}0.76}$ & $\mathbf{51.88{\pm}0.19}$ \\ \bottomrule
\end{tabular}
}
\end{table}

BLADE obtains $51.88\pm0.19$ average accuracy, compared with $50.65\pm0.26$ for RHO-1, and has the highest mean on seven of the eight tasks. These repeated runs show that the improvement in Table~\ref{tab:sota_general} is stable across random seeds rather than being driven by a single run.

\subsection{Training Dynamics in General Continual Pre-training}
\label{sec:general_dynamics}

To check whether the final improvement appears only at the last checkpoint, we evaluate the general-pretraining runs after every 1B candidate tokens processed. Here, \emph{processed candidate tokens} counts all tokens inspected by the selection pipeline, whereas \emph{training tokens} counts only retained tokens whose losses contribute to parameter updates. Base-CT trains on every processed token. The selection methods use the same keeping ratio $\gamma=0.6$. Consequently, processing the full 5B-token candidate pool corresponds to 3B actual training tokens for each selection method. As shown in Figure~\ref{fig:general_curve}, RHO-1 is slightly ahead of BLADE after the first 1B processed tokens, but BLADE overtakes it at 2B and remains the strongest method shown thereafter. The margin over RHO-1 grows from 0.67 points at 2B to 1.15 points at 5B. 

\begin{figure}[h]
\centering
\includegraphics[width=0.5\textwidth]{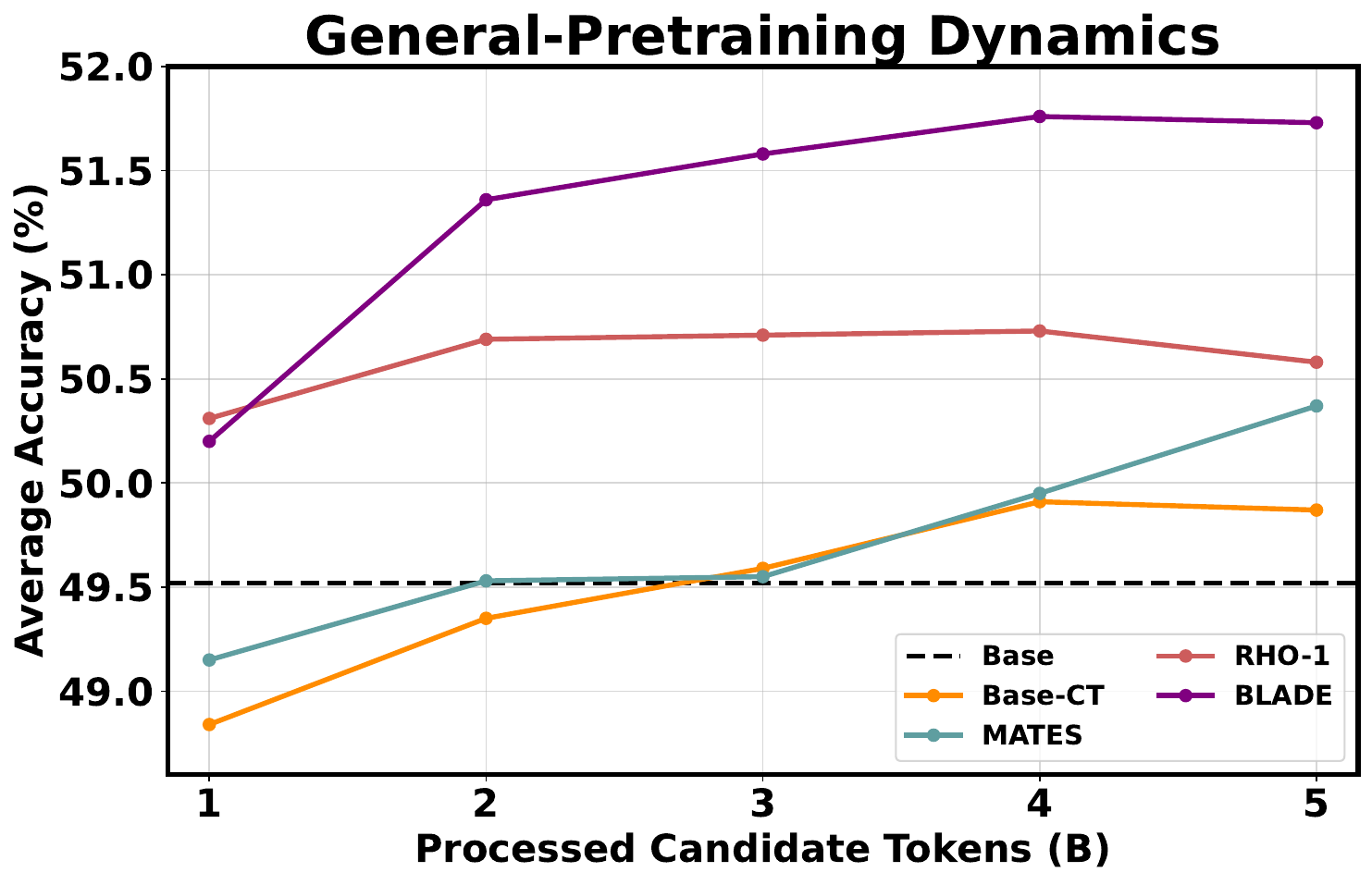}
\caption{Average general-benchmark accuracy over training. Base-CT trains on every processed token, whereas selection methods retain $0.6\times$ as many tokens for optimization.}
\label{fig:general_curve}
\end{figure}

\subsection{The Transferability of the Selected Tokens}
\begin{table}[]
\caption{Comparison for the CoT reasoning results of math pre-training on the Llama2-7B. $\dagger$ denotes the results using our implementation. $*$ means sample level selection method. The result of transferring is grayed, where the token orchestration learned with TinlyLlama-1.1B is transferred to a larger (7B) model training.}
\label{tab:token_transferablity}
\centering
\footnotesize 
\begin{tabular}{@{}lc|cccccccccc@{}}
\toprule
         & \begin{tabular}[c]{@{}c@{}}Train\\ Toks\end{tabular} & GSM8K & MATH & SVAMP & ASDiv & MAWPS & TAB  & MQA  & \begin{tabular}[c]{@{}c@{}}MMLU\\ STEM\end{tabular} & SAT  & Avg.  \\ \midrule
Base     & -                                                    & 14.0  & 3.6  & 39.5  & 51.7  & 63.5  & 30.9 & 12.4 & 32.7                                                & 34.4 & 31.4 \\
Base-CT  & 5B                                                   & 18.1  & 4.4  & 36.4  & 52.7  & 66.1  & 37.3 & 18.5 & 33.4                                                & 37.5 & 33.8 \\
Random   & 3B                                                   & 17.3  & 5.2  & 36.9  & 51.9  & 64.2  & 40.2 & 17.1 & 32.9                                                & 25.0 & 32.3 \\
MATES$^*$    & 3B                                                   & 16.5  & 4.4  & 35.5  & 51.6  & 65.7  & 38.6 & 15.8 & 33.6                                                & 38.6 & 33.4 \\
RHO-1$\dagger$    & 3B                                                   & 33.0  & 7.6  & 57.9  & 60.6  & 78.9  & 48.4 & 24.9 & 30.1                                                & 28.1 & 41.1 \\
BLADE    & 3B                                                   & \textbf{40.4}  & \textbf{9.8}  & \textbf{60.8}  & \textbf{66.1}  & \textbf{84.9}  & \textbf{49.6} & \textbf{27.3} & 33.1                                                & 21.9 & \textbf{43.8} \\
\rowcolor{gray!20} BLADE-Tr & 3B                                                   & 23.3  & 5.8  & 41.5  & 56.4  & 68.2  & 42.5 & 21.0 & \textbf{35.3}                                                & \textbf{40.6} & 37.2 \\ \bottomrule
\end{tabular}
\end{table}

While BLADE effectively identifies high-quality training data, scoring and selecting data natively on massive LLMs introduces computational overhead. To further reduce the selection computation, we investigate whether data selected by a smaller proxy model is transferable. We train the 7B-parameter Llama2 model using the exact token pre-selected by BLADE on the 1.1B TinyLlama (denoted as BLADE-Tr). As shown in Table~\ref{tab:token_transferablity}, BLADE-Tr achieves an average accuracy of 37.2, substantially outperforming the random sampling baseline (32.3) and naive continuous pre-train baseline (33.8). This indicates that data quality signals are partially model-agnostic and transferable. However, BLADE-Tr noticeably underperforms native BLADE on the 7B model (43.8). This gap demonstrates that optimal data selection is capacity-dependent. The 1.1B proxy model may misclassify "hard but learnable" samples as unlearnable noise, discarding data that would actually be highly informative for the 7B model. Therefore, while cross-scale data transfer is a viable low-cost compromise, fully unlocking a large model's potential requires capacity-matched data selection and orchestration.

\section{Summary of Notations}
\label{sec:summary_of_notations}

\begin{table}[H]
\caption{Summary of the notations used throughout this paper. Variables only used in theoretical analysis are grayed for better readability. }
\setlength\tabcolsep{2.5 pt}
\begin{tabular}{@{}lll@{}}
\toprule
Topic           & Notation          & Explanation             \\ \midrule
Data Sets       & $N$               & The number of candidate selection units (tokens here).           \\
                & $\gamma$           & The keeping ratio.           \\
                & $\cD_{\text{train}}$    & Train set.               \\
                & $\cD_{\text{val}}$      & Validation set.                                                                \\ \midrule
Models \& Parameters & $\bu$             & Parameters of the proxy model.                                                                \\
                      & $\bw$                                                           & Parameters of the reference model.                                                  \\
\rowcolor{gray!40}    & $\bw^{*}$                                                       & Optimal solution for the lower level problem.                                                  \\
\rowcolor{gray!40}    & $\bw_{\lambda}^*$     & Optimal solution for the penalized problem.                                             \\
\rowcolor{gray!40}   & $\bS^*(\boldsymbol{\alpha})$                                                               & Solution set for the  lower level problem. \\
\rowcolor{gray!40}   & $\bS_\lambda^*(\boldsymbol{\alpha})$                                                               & Solution set for the  penalized problem.                                           \\
                     & $\balpha$                                                    & Token-selection weight vector.                                           \\
                     & $\cA$                                                    & The feasible set of $\balpha$.                                           \\ \midrule
Problems \& Losses        & $\mathcal{L}_{\text{train}}^n$        &  Train loss on the $n$-th training sample.      \\  

                                   &  $\mathcal{L}_{\text{train}}$        &   Overall train loss weighted by the selection weight $\balpha$.  \\
                                   &  $\mathcal{L}_{\text{val}}$        &   Overall validation loss.  \\
\rowcolor{gray!40}                 &  $\cH(\balpha)$        &   The original bi-level optimization problem.  \\
\rowcolor{gray!40}                 &  $\cH_{\lambda}(\balpha, \bw)$        &   The penalized problem.  \\
\rowcolor{gray!40}                 &  $\cH_{\lambda}(\balpha, \bw_{\lambda}(\bw^*))$        &   The penalized problem with $\bw$ optimized.  \\
\rowcolor{gray!40}                 &   $G(\balpha)$                       &   The Frank-Wolfe Gap of the original bi-level optimization problem.  \\
\rowcolor{gray!40}                 &   $G_{\lambda}(\balpha)$             &   The Frank-Wolfe Gap of the penalized problem with $\bw$ optimized.  \\
                                   
                                   \midrule
\rowcolor{gray!40} Function Properties        & $\mu$     & PL coefficient of the lower-level problem.        \\
\rowcolor{gray!40}           & $\mu_{\gamma}$                                                          &  PL coefficient of the penalized problem. \\ 
\rowcolor{gray!40}           & $L$                                                          &  Lipschitz constant for $\nabla_{\boldsymbol{w}} \mathcal{L}_{\text {train }}(\boldsymbol{\alpha}, \boldsymbol{w})$ on $\bw$. \\
\rowcolor{gray!40}           & $B$                                                          &  Lipschitz constant for $\mathcal{L}_{\text {train }}(\boldsymbol{\alpha}, \boldsymbol{w})$ and $\mathcal{L}_{\text {val }}(\boldsymbol{w})$ on $\bw$. \\
\rowcolor{gray!40}           & $\kappa$ and $\rho$                                                          &  Hessian $\nabla^2_{\boldsymbol{w} \boldsymbol{w}} \mathcal{L}_{\text {train }}(\boldsymbol{\alpha}, \boldsymbol{w}) \succeq \kappa$ and $\nabla_{\boldsymbol{\alpha} \boldsymbol{w}}^2 \mathcal{L}_{\text {train }}(\boldsymbol{\alpha}, \boldsymbol{w}) \preceq \rho$ \\
\rowcolor{gray!40}           & $D$                                                          &  Upper bound for train/validation loss on the trajectory region. \\
\rowcolor{gray!40}           & $L_\lambda$                                                          &  Lipschitz constant for $\nabla_{\boldsymbol{\alpha}} \mathcal{H}_\lambda\left(\boldsymbol{\alpha}, \boldsymbol{w}_\gamma^*(\boldsymbol{\alpha})\right)$.  \\
\rowcolor{gray!40}           & $D_{\mathcal{A}}$                                   &    Diameter of the feasible set of $\balpha$. \\
\midrule
Train                 & $t$                                                               & Proxy model $\bu$ training step/Selection step                                                           \\
                      & $T$                                                               & The total number of proxy model training steps.                                    \\ 
                       & $k$                                                               &  $\bw$ update step.  \\
                        & $K$                                                               & The number of $\bw$ update steps each episode. \\
                       & $\tau$                                                               & Episode of free $\bu$ update step.  \\
                      & $\eta_{\balpha}$                                                               & The learning rate on $\balpha$.  \\
                      & $\eta_{\bu}$                                                               & The learning rate on $\bu$.  \\
                      & $\eta_{\bw}$                                                               & The learning rate on $\bw$.  \\
                      & $\lambda$                                                               & The penalty strength.  \\
                      \bottomrule
\end{tabular}
\end{table}

\section{Limitations and Future Works}
\label{sec:limitations}
Despite the advancements introduced in this work, several challenges remain open. First, our experiments extend to 13B parameters and 10B processed tokens, but frontier-scale models with hundreds of billions of parameters remain outside our academic compute budget. Second, as with other validation-guided approaches, BLADE benefits from a high-quality and representative target validation set. We use a static, human-curated blend; automatically constructing and updating this distribution is a promising direction. Third, Theorem~\ref{thm:convergence} analyzes the full-coordinate selection criterion. The practical single-pass algorithm is a randomized block-coordinate realization whose dedicated stochastic convergence analysis remains future work. These limitations concern scale, validation-set construction, and the formal treatment of the batched implementation rather than the applicability of the selection rule itself.

%% file: bibfile.bib
@inproceedings{mindermann2022rho,
  title = 	{Prioritized Training on Points that are Learnable, Worth Learning, and Not Yet Learnt},
  author =    {Sören Mindermann and Jan Brauner and Muhammed Razzak and Mrinank Sharma and Andreas Kirsch and Winnie Xu and Benedikt Höltgen and Aidan N. Gomez and Adrien Morisot and Sebastian Farquhar and Yarin Gal},
  booktitle = {International Conference on Machine Learning},
  year = 	 {2022}
}

@inproceedings{lin2024rho1,
title={RHO-1: Not All Tokens Are What You Need},
author={Zhenghao Lin and Zhibin Gou and Yeyun Gong and Xiao Liu and Yelong Shen and Ruochen Xu and Chen Lin and Yujiu Yang and Jian Jiao and Nan Duan and Weizhu Chen},
booktitle={Neural Information Processing Systems},
year={2024},
}

@inproceedings{yu2024mates,
title={MATES: Model-Aware Data Selection for Efficient Pretraining with Data Influence Models},
author={Zichun Yu and Spandan Das and Chenyan Xiong},
booktitle={Neural Information Processing Systems},
year={2024},
}

@InProceedings{pan2024gdig,
  title = 	 {G-DIG: Towards Gradient-based Diverse and High-quality Instruction Data Selection for Machine Translation},
  author =       {Xingyuan Pan and Luyang Huang and Liyan Kang and Zhicheng Liu and Yu Lu and Shanbo Cheng},
  booktitle = 	 {Association for Computational Linguistics},
  year = 	 {2024}
}

@inproceedings{koh2017influence,
  title = 	{Understanding Black-box Predictions via Influence Functions},
  author =    {Pang Wei Koh and Percy Liang},
  booktitle = {International Conference on Machine Learning},
  year = 	 {2017}
}

@inproceedings{shen23onpenalty,
  title = 	{On Penalty-based Bilevel Gradient Descent Method},
  author =    {Shen, Han and Chen, Tianyi},
  booktitle = {International Conference on Machine Learning},
  year = 	 {2023}
}

@InProceedings{kwon2023f2sa,
  title = 	 {A Fully First-Order Method for Stochastic Bilevel Optimization},
  author =       {Jeongyeol Kwon and Dohyun Kwon and Stephen Wright and Robert Nowa},
  booktitle = 	 {International Conference on Machine Learning},
  year = 	 {2023}
}

@inproceedings{wang2025tandem,
title={TANDEM: Bi-Level Data Mixture Optimization with Twin Networks},
author={Jiaxing Wang and Deping Xiang and Jin Xu and Mingyang Yi and Guoqiang Gong and Zicheng Zhang and Haoran Li and Pengzhang Liu and Zhen Chen and Ke Zhang and Ju Fan and Qixiang Jiang},
booktitle={Neural Information Processing Systems},
year={2025},
}

@inproceedings{wang2024greats,
title={GREATS: Online Selection of High-Quality Data for
LLM Training in Every Iteration},
author={Jiachen T. Wang and Tong Wu and Dawn Song and Prateek Mittal and Ruoxi Jia},
booktitle={Neural Information Processing Systems},
year={2024},
}

@InProceedings{xia2024less,
  title = 	 {LESS: Selecting Influential Data for Targeted Instruction Tuning},
  author =     {Mengzhou Xia and Sadhika Malladi and Suchin Gururangan and Sanjeev Arora and Danqi Chen},
  booktitle = 	 {International Conference on Machine Learning},
  year = 	 {2024}
}

@inproceedings{choe2024metalearning,
 author = {Sang Keun Choe and Sanket Vaibhav Mehta and Hwijeen Ahn, Willie Neiswanger and Pengtao Xie and Emma Strubell and Eric Xing},
 booktitle = {Advances in Neural Information Processing Systems},
 title = {Making Scalable Meta Learning Practical},
 year = {2024}
}

@InProceedings{lacoste-julien2013blockfw,
  title = 	 {Block-Coordinate {Frank-Wolfe} Optimization for Structural {SVMs}},
  author = 	 {Lacoste-Julien, Simon and Jaggi, Martin and Schmidt, Mark and Pletscher, Patrick},
  booktitle = 	 { International Conference on Machine Learning},
  year = 	 {2013},
}

@InProceedings{reddi2016sfwnoncovex,
  title = 	 {Stochastic Frank-Wolfe Methods for Nonconvex Optimization},
  author = 	 {Sashank J. Reddi and Suvrit Sra and Barnabas Poczos and Alex Smola},
  booktitle = 	 {Annual Allerton Conference on Communication, Control, and Computing},
  year = 	 {2016},
}

@InProceedings{hazan2012prejectionfree,
  title = 	 {Projection-free Online Learning},
  author = 	 {Elad Hazan and Satyen Kale},
  booktitle = 	 { International Conference on Machine Learning},
  year = 	 {2012},
}

@inproceedings{yu2024metamath,
title={MetaMath: Bootstrap Your Own Mathematical Questions for Large Language Models},
author={Longhui Yu and Weisen Jiang and Han Shi and Jincheng Yu and Zhengying Liu and Yu Zhang and James T. Kwok and Zhenguo Li and Adrian Weller and Weiyang Liu},
booktitle={International Conference on Learning Representations},
year={2024},
}

@inproceedings{yu22024mammoth,
title={Mammoth: Building Math Generalist Models Through Hybrid Instruction Tuning},
author={Xiang Yue and Xingwei Qu and Ge Zhang and Yao Fu and Wenhao Huang and Huan Sun and Yu Su and and Wenhu Chen},
booktitle={International Conference on Learning Representations},
year={2024},
}

@article{zhang2024tinyllama,
  author       = {Peiyuan Zhang and Guangtao Zeng and Tianduo Wang and Wei Lu},
  title        = {Tinyllama: An open-source small language model},
  journal      = {arXiv preprint},
  volume       = {arXiv:2401.02385},
  year         = {2024}
}

@inproceedings{paster2024openwebmath,
title={OpenWebMath: An Open Dataset of High-Quality Mathematical Web Text.},
author={Keiran Paster and Marco Dos Santos and Zhangir Azerbayev and Jimmy Ba},
booktitle={International Conference on Learning Representations},
year={2024},
}

@inproceedings{dagreou2022bilevel,
 author = {Dagréou Mathieu and Pierre Ablin and Samuel Vaiter and Thomas Moreau},
 booktitle = {Advances in Neural Information Processing Systems},
 title = {A framework for bilevel optimization that enables stochastic and global variance reduction algorithms},
 year = {2022}
}

@InProceedings{shaban2019tuncated,
  title = 	 {Truncated Back-propagation for Bilevel Optimization},
  author =       {Amirreza Shaban and Ching-An Cheng and Nathan Hatch and Byron Boots},
  booktitle = 	 {International Conference on Artificial Intelligence and Statistics},
  year = 	 {2019}
}

@inproceedings{grazzi2023bilevel,
 author = {Riccardo Grazzi and Massimiliano Pontil and Saverio Salzo},
 booktitle = {Journal of Machine Learning Research },
 title = {Bilevel Optimization with a Lower-level Contraction: Optimal Sample Complexity without Warm-Start},
 year = {2023}
}

@inproceedings{chen2021tight,
 author = {Chen, Tianyi and Sun, Yuejiao and Yin, Wotao},
 booktitle = {Advances in Neural Information Processing Systems},
title = {Closing the Gap: Tighter Analysis of Alternating Stochastic Gradient Methods for Bilevel Problems},
 year = {2021}
}

@article{hong2023twotimescale,
  author={Mingyi Hong and Hoi-To Wai and Zhaoran Wang and Zhuoran Yang},
  title={A Two-Timescale Stochastic Algorithm Framework for Bilevel Optimization: Complexity Analysis and Application to Actor-Critic},
  year={2023},
  journal={SIAM J. Optim.}
}

@inproceedings{kwon2024onpenalty,
title={On Penalty Methods for Nonconvex Bilevel Optimization and First-Order Stochastic Approximation},
author={Jeongyeol Kwon and Dohyun Kwon and Stephen Wright and Robert D Nowak},
booktitle={International Conference on Learning Representations},
year={2024},
}

@inproceedings{zhou2022pbilevel,
  title = 	{Probabilistic Bilevel Coreset Selection},
  author =    {Xiao Zhou and Renjie Pi and Weizhong Zhang and Yong Lin and Tong Zhang},
  booktitle = {International Conference on Machine Learning},
  year = 	 {2022}
}

@article{teknium2023openhermes,
  author       = {Teknium},
  title        = {Openhermes 2.5: An Open Dataset of Synthetic Data for Generalist LLM Assistants},
  url      = {https://huggingface.co/datasets/teknium/OpenHermes-2.5},
  year         = {2023}
}

@article{hendrycks2020mmlu,
  author       = {Dan Hendrycks and Collin Burns and Steven Basart and Andy Zou and Mantas Mazeika and Dawn Song and Jacob Steinhardt},
  title        = {Measuring Massive Multitask Language Understanding},
  journal      = {arXiv preprint},
  volume       = {arXiv:2009.03300},
  year         = {2020}
}

@article{clark2018arc,
  author       = {Peter Clark and Isaac Cowhey and Oren Etzioni and Tushar Khot and Ashish Sabharwal and Carissa Schoenick and Oyvind Tafjord},
  title        = {Think You Have Solved Question Answering? Try ARC, the AI2 Reasoning Challenge.},
  journal      = {arXiv preprint},
  volume       = {arXiv:1803.05457},
  year         = {2018}
}

@article{clark2019boolq,
  author       = {Christopher Clark abd Kenton Lee and Ming-Wei Chang and Tom Kwiatkowski and Michael Collins and Kristina Toutanova.},
  title        = {Boolq: Exploring the Surprising Difficulty of Natural Yes/No Questions},
  journal      = {arXiv preprint},
  volume       = {arXiv:1905.10044},
  year         = {2019}
}

@inproceedings{bisk2020pika,
  title = 	{Piqa: Reasoning About Physical Commonsense in Natural Language},
  author =    {Yonatan Bisk and Rowan Zellers and Ronan Le Bras and Jianfeng Gao and Yejin Choi},
  booktitle = {AAAI conference on artificial intelligence},
  year = 	 {2020}
}

@article{zellers2018hellaswag,
  author       = {Rowan Zellers and Ari Holtzman and Yonatan Bisk and Ali Farhadi and Yejin Choi},
  title        = {Hellaswag: Can a machine really finish your sentence? },
  journal      = {arXiv preprint},
  volume       = {arXiv:1905.07830},
  year         = {2019}
}

@inproceedings{sakaguchi2020winogrande,
  title = 	{Winogrande: An Adversarial Winograd Schema Challenge at Scale},
  author =    {Keisuke Sakaguchi and Ronan Le Bras and Chandra Bhagavatula and Yejin Choi},
  booktitle = {AAAI conference on artificial intelligence},
  year = 	 {2020}
}

@article{2018openbookqa,
  author       = {Todor Mihaylov and Peter Clark and Tushar Khot and Ashish Sabharwal},
  title        = {Can A Suit of Armor Conduct Electricity? A New Dataset for Open Book Question Aanswering},
  journal      = {arXiv preprint},
  volume       = {arXiv:1809.02789},
  year         = {2018}
}

@article{cobbe2021gsm8k,
  author       = {Karl Cobbe and Vineet Kosaraju and Mohammad Bavarian and Mark Chen and Heewoo Jun and Lukasz Kaiser and Matthias Plappert and Jerry Tworek and Jacob Hilton and Reiichiro Nakano and Christopher Hesse and John Schulman},
  title        = {Training verifiers to solve math word problems},
  journal      = {arXiv preprint},
  volume       = {arXiv:2110.14168},
  year         = {2021}
}

@inproceedings{hendrycks2021math,
title={Measuring mathematical problem solving with the math dataset},
author={Dan Hendrycks and Collin Burns and Saurav Kadavath and Akul Arora and Steven Basart and Eric Tang and Dawn Song and Jacob Steinhardt},
booktitle={Neural Information Processing Systems},
year={2021},
}

@inproceedings{patel2021svamp,
title={Are NLP models really able to solve simple math word problems?},
author={Arkil Patel and Satwik Bhattamishra and Navin Goyal},
booktitle={Conference of the North American Chapter of the Association for Computational Linguistics},
year={2021},
}

@InProceedings{miao2020asdiv,
  title = 	 {A diverse corpus for evaluating and developing English math word problem solvers},
  author =       {Shenyun Miao and Chaochun Liang and Keh-Yih Su},
  booktitle = 	 {Association for Computational Linguistics},
  year = 	 {2020}
}

@InProceedings{kedziorski2016mawps,
  title = 	 {MAWPS: A math word problem repository},
  author =       {Rik Koncel Kedziorski and Subhro Roy and Aida Amini and Nate Kushman and Hannaneh Hajishirzi},
  booktitle = 	 {Conference of the North American Chapter of the Association for Computational Linguistics},
  year = 	 {2016}
}

@inproceedings{lu2023tab,
title={Dynamic prompt learning via policy gradient for semi-structured mathematical reasoning},
author={Pan Lu and Liang Qiu and Kai-Wei Chang and Ying Nian Wu and Song-Chun Zhu and Tanmay Rajpurohit and Peter Clark and Ashwin Kalyan},
booktitle={International Conference on Learning Representations},
year={2023},
}

@article{amini2019mathqa,
  author       = {Aida Amini and Saadia Gabriel and Peter Lin and Rik Koncel-Kedziorski and Yejin Choi and Hannaneh Hajishirzi},
  title        = {MATHQA: Towards Interpretable Math Word Problem Solving with Operation-based Formalisms},
  journal      = {arXiv preprint},
  volume       = {arXiv:1905.13319},
  year         = {2019}
}

@article{soboleva2023slimpajama,
  author       = {Daria Soboleva and Faisal Al-Khateeb and Joel Hestness and Nolan Dey
Opentensor: Robert Myers, Jacob Robert Steeves},
  title        = {SlimPajama: A 627B token, cleaned and deduplicated version of RedPajam},
  journal = {https://www.cerebras.ai/blog/slimpajama-a-627b-token-cleaned-and-deduplicated-version-of-redpajama},
  year         = {2023}
}

@misc{gao2024lmevalharness,
  author       = {Gao, Leo and Tow, Jonathan and Abbasi, Baber and Biderman, Stella and Black, Sid and DiPofi, Anthony and Foster, Charles and Golding, Laurence and Hsu, Jeffrey and Le Noac'h, Alain and Li, Haonan and McDonell, Kyle and Muennighoff, Niklas and Ociepa, Chris and Phang, Jason and Reynolds, Laria and Schoelkopf, Hailey and Skowron, Aviya and Sutawika, Lintang and Tang, Eric and Thite, Anish and Wang, Ben and Wang, Kevin and Zou, Andy},
  title        = {The Language Model Evaluation Harness},
  year         = 2024,
  url          = {https://zenodo.org/records/12608602}
}

@inproceedings{raffel2020T5,
  title={Exploring the limits of transfer learning with a unified text-to-text transformer},
  author={Colin Raffel and Noam Shazeer and Adam Roberts and Katherine Lee and Sharan Narang and Michael Matena and Yanqi Zhou and Wei Li and Peter J. Liu},
  booktitle={Journel of Machine Learning Research},
  year={2020}
}

@article{rae2021gopher,
  author       = {Jack W Rae and Sebastian Borgeaud and Trevor Cai and Katie Millican and Jordan Hoffmann and Francis Song and John Aslanides and Sarah Henderson and Roman Ring and Susannah Young et. al.},
  title        = {Scaling language models: Methods, analysis \& insights from training Gopher},
  journal      = {arXiv preprint},
  volume       = {arXiv:2112.11446},
  year         = {2021}
}

@inproceedings{tirumala2023d4,
title={D4: Improving LLM pretraining via document de-duplication and diversification},
author={Kushal Tirumala and Daniel Simig and Armen Aghajanyan and Ari Morcos},
booktitle={Neural Information Processing Systems},
year={2023},
}

@inproceedings{xie2023dsir,
title={Data selection for language models via importance resampling},
author={Sang Michael Xie and Shibani Santurkar and Tengyu Ma and Percy Liang},
booktitle={Neural Information Processing Systems},
year={2023},
}

@article{hao2026bliss,
  author       = {Jie Hao and Rui Yu and Wei Zhang and Huixia Wang and Jie Xu and Mingrui Liu},
  title        = {BLISS: A Lightweight Bilevel Influence Scoring Method for Data Selection in Language Model Pretraining},
  journal      = {arXiv preprint},
  volume       = {arXiv:2510.06048},
  year         = {2025}
}

@article{neyshabur14implicitreg,
  author       = {Behnam Neyshabur and Ryota Tomioka and Nathan Srebro},
  title        = {In Search of the Real Inductive Bias: On the Role of Implicit Regularization in Deep Learning},
  journal      = {arXiv preprint},
  volume       = {arXiv:1412.6614},
  year         = {2014}
}

@inproceedings{soudry18implicitbias,
title={The Implicit Bias of Gradient Descent on Separable Data},
author={Daniel Soudry and Elad Hoffer and Mor Shpigel Nacson and Suriya Gunasekar and Nathan Srebro},
booktitle={Journal of Machine Learning Research},
year={2018},
}

@article{zhou23lima,
  author       = {Chunting Zhou and Pengfei Liu and Puxin Xu and Srini Iyer and Jiao Sun and Yuning Mao and Xuezhe Ma and Avia Efrat and Ping Yu and Lili Yu and Susan Zhang and Gargi Ghosh and Mike Lewis and Luke Zettlemoyer and Omer Levy},
  title        = {LIMA: Less Is More for Alignment},
  journal      = {arXiv preprint},
  volume       = {arXiv:2305.11206},
  year         = {2023}
}

@article{touvron2023llama2,
  author       = {Hugo Touvron and Louis Martin and Kevin Stone and Peter Albert and Amjad Almahairi and Yasmine Babaei and Nikolay Bashlykov et.al.},
  title        = {Llama 2: Open Foundation and Fine-Tuned Chat Models},
  journal      = {arXiv preprint},
  volume       = {arXiv:2307.09288},
  year         = {2023}
}

@inproceedings{engstrom2024dsdm,
  title={DsDm: Model-Aware Dataset Selection with Datamodels},
  author={Engstrom, Logan and Feldmann, Axel and Madry, Aleksander},
  booktitle={International Conference on Machine Learning},
  pages={12491--12526},
  year={2024},
  organization={PMLR}
}

@article{simon2016convergefw,
  author       = {Simon Lacoste-Julien},
  title        = {Convergence Rate of Frank-Wolfe for Non-Convex Objectives},
  journal      = {arXiv preprint},
  volume       = {arXiv:1607.00345},
  year         = {2016}
}

@article{azerbayev2023sat,
  author       = {Zhangir Azerbayev and Hailey Schoelkopf and Keiran Paster and Marco Dos Santos and Stephen McAleer and Albert Q. Jiang and Jia Deng and Stella Biderman and Sean Welleck},
  title        = {Llemma: An Open Language Model For Mathematics},
  journal      = {arXiv preprint},
  volume       = {arXiv:2310.10631},
  year         = {2023}
}

@inproceedings{pang2025tokencleaning,
  title={Token Cleaning: Fine-Grained Data Selection for LLM Supervised Fine-Tuning},
  author={Jinlong Pang and Na Di and Zhaowei Zhu and Jiaheng Wei and Hao Cheng and Chen Qian and Yang Liu},
  booktitle={International Conference on Machine Learning},
  year={2025},
}

@article{ye2022bome,
  title={BOME! Bilevel Optimization Made Easy: A Simple First-Order Approach},
  author={Mao Ye and Bo Liu and Stephen Wright and Peter Stone and Qiang Liu },
  journal={Advances in Neural Information Processing Systems},
  year={2022}
}
